\documentclass[12pt]{alt2022} % Include author names
\usepackage{smile}
\usepackage{Command}
\allowdisplaybreaks
\title[Almost Optimal Algorithms for Two-player Zero-Sum Linear Mixture Markov Games]{Almost Optimal Algorithms for Two-player Zero-Sum Linear Mixture Markov Games}
\usepackage{times}
\altauthor{%
 \Name{Zixiang Chen} \Email{chenzx19@cs.ucla.edu}\\
 \addr Department of Computer Science, University of California, Los Angeles, CA 90095, USA
 \AND
 \Name{Dongruo Zhou} \Email{drzhou@cs.ucla.edu}\\
 \addr Department of Computer Science, University of California, Los Angeles, CA 90095, USA%
 \AND
 \Name{Quanquan Gu} \Email{qgu@cs.ucla.edu}\\
 \addr Department of Computer Science, University of California, Los Angeles, CA 90095, USA%
}

\begin{document}

\maketitle

\begin{abstract}%
We study reinforcement learning for two-player zero-sum Markov games with simultaneous moves in the finite-horizon setting, where the transition kernel of the underlying Markov games
can be parameterized by a linear function over the current state, both players' actions and the next state. In particular, we assume that we can control both players and aim to find the Nash Equilibrium by minimizing the duality gap. We propose an algorithm Nash-UCRL based on the principle ``Optimism-in-Face-of-Uncertainty''. Our algorithm only needs to find a Coarse Correlated Equilibrium (CCE), which is computationally efficient. Specifically, we show that Nash-UCRL can provably achieve an $\tilde{O}(dH\sqrt{T})$ regret, where $d$ is the linear function dimension, $H$ is the length of the game and $T$ is the total number of steps in the game. To assess the optimality of our algorithm, we also prove an $\tilde{\Omega}( dH\sqrt{T})$ lower bound on the regret. Our upper bound matches the lower bound up to logarithmic factors, which suggests the optimality of our algorithm.
\end{abstract}


\section{Introduction}
Multi-agent reinforcement learning (MARL) has achieved tremendous practical success across a wide range of machine learning tasks, including large-scale strategy games such as GO \citep{silver2016mastering}, TexasHold’em poker \citep{brown2019superhuman}, real-time video games such as Starcraft \citep{vinyals2019grandmaster}, and autonomous driving \citep{shalev2016safe}. Among these models used in MARL, two-player zero-sum Markov games (MG) \citep{shapley1953stochastic,littman1994markov} is probably one of the most widely studied models and can be regarded as a generalization of the Markov Decision Processes (MDP) \citep{puterman2014markov}.

% (Important to understand sample complexity) 
% (Important to understand sample complexity)Despite the great success of reinforcement learning, many real-world settings are sample-expensive such as robotics and autonomous driving.  Thus it is crucial to theoretically understand the sample complexity in RL, i.e., the minimum number of samples required for any algorithm to find an acceptable policy and an algorithm to find the near-optimal policy. 

%However, the success of MARL has not been well explained theoretically, especially compared with the theoretical understanding of single-agent RL. 
In two-player Markov games, the two players share states, play actions simultaneously and independently, and observe the same reward. One player (i.e., max-player) aims to maximize the return while the other (i.e., min-player) aims to minimize it. %has $S$ states and $A$ actions for one player, $B$ actions for the other player, and $H$ stages per episode. 
A special case of general Markov games (i.e., simultaneous-move games) is turn-based games, where only one player can take action in each step, i.e., the max and min players take turns to play the game. The players aim to find the Nash equilibrium for this game. %The focus of this paper is to design low-regret algorithms for solving episodic two-player Markov games in the general setting \citep{kearns2002near}. 
Most existing results on learning two-player Markov games either assume the access to a generative model that can sample the next state for an arbitrary state-action pair \citep{jia2019feature,sidford2020solving,cui2020minimax}, or a well-explored behavior policy \citep{lagoudakis2012value,perolat2015approximate,perolat2016softened,perolat2016use,perolat2017learning}, and fail to consider the exploration-exploitation tradeoff \citep{kearns2002near}. 

In order to get rid of the generative model and well-explored behavior policy assumptions, 
\citet{wei2017online} extended the UCRL2 algorithm \citep{jaksch2010near} for MDP to zero-sum simultaneous-move Markov games in the average-reward setting, and proposed the UCSG algorithm that achieves a sublinear regret when competing with an arbitrary opponent. %$\tilde O(S\sqrt{ABT})$ regret. 
Recently, \citet{bai2020provable,bai2020near,liu2020sharp} proposed a series of algorithms for learning tabular episodic two-player zero-sum Markov games (they call it self-play algorithm for competitive reinforcement learning), and proved the upper and lower regret bounds and/or sample complexity.
\iffalse
\citet{bai2020provable} proposed a VI-ULCB algorithm for tabular episodic zero-sum Markov games, which achieves $\tilde O(\sqrt{H^{3}S^{2}ABT})$ regret for simultaneous move (i.e., general Markov game) and $O(\sqrt{H^{3}S^{2}(A+B)T})$ regret for turn-based game, where $A$ and $B$ are the number of actions for each player, $H$ is the length of the game, and $T$ is the total number of steps played in the game. They also proved a $\Omega(\sqrt{H^{2}S(A+B)T})$ lower bound. %\todoq{Is bai's setting the same as Wei? Also, I remember there are two papers by Bai and Jin. What about the other one?} 
For general Markov game, \citet{bai2020near} proposed an Optimistic Nash Q-learning algorithm with a regret of $\tilde O(\sqrt{H^{4}SABT})$, and an Optimistic Nash V-learning algorithm with a regret of $\tilde O(\sqrt{H^{5}S(A+B)T})$, both of which improve the regret in \citep{bai2020provable} in the dependence on $S,A,B$.
% One of the best known algorithm is VI-ULCB, the regret to find an approximate Nash equilibrium in $\tilde O(\sqrt{H^{3}S^{2}ABT})$\citep{bai2020provable}. 
The best known regret is achieved by Nash-VI proposed in \citep{liu2020sharp}, which is $\tilde O(\sqrt{H^{2}SABT})$. 
\fi
For Markov games with large state and action spaces, it is natural to use linear function approximation. In particular,
\citet{xie2020learning} proposed the OMNI-VI algorithm for Markov games where the transition kernel and reward function 
possess a linear structure, and achieved an $\tilde O(\sqrt{d^{3}H^{3}T})$ regret bound, with $d$ being the dimension of the linear structure and $H$ being the episode length. However, as we will show in this paper, the information theoretic lower bound for the zero-sum two-player Markov games with linear structures is $\Omega(dH\sqrt{T})$. Therefore, there is still a gap between the upper and lower bounds of  existing algorithms %are minimax optimal for either tabular Markov games or 
for Markov games with linear structures. This raises the following question:

\begin{center}
\emph{
Can we design a minimax optimal algorithm for learning zero-sum Markov games with linear function approximation?
}    
\end{center}

In this paper,  we give an affirmative answer to the above question for a class of episodic Markov games in the offline setting\footnote{Here we follow the same terminology ``offline setting'' as in \citet{xie2020learning}, which is also called ``self-play'' in \citet{bai2020provable,bai2020near}}, where both players are controlled by a central learner. The goal of the central learner is to find an approximate Nash Equilibrium (NE) of the game, with the approximation error measured by a notion of duality gap. In particular, we consider Markov games with a linear mixture structure, where the transition probability kernel is a linear mixture model that is inspired by the linear mixture MDPs studied in \citep{modi2019sample,jia2020model,ayoub2020model,zhou2020provably}. We propose the first nearly minimax optimal algorithm based on the principle of ``Optimism-in-Face-of-Uncertainty'' without assuming the access to the generative model or well-explored behavior policy. %To be specific, our algorithm is Nash UCRL based on the principle “Optimism-in-Face-of-Uncertainty .” 
We summarize the contributions of our work as follows:

\begin{itemize}
\item We propose a $\algname$ algorithm for general Markov games (i.e., simultaneous-move game) that can provably achieve an $\tilde{O}(dH\sqrt{T})$ upper bound on the regret, where $d$ is the dimension of linear mixture structure, $H$ is the length of the game, and $T$ the total number of steps in the Markov game. Our algorithm can be specialized to turn-based games and also achieves $\tilde{O}(dH\sqrt{T})$ regret.
\item To access the optimality of our algorithm $\algname$, we prove an $\Omega( dH\sqrt{T})$ regret lower bound . Our upper bound matches the lower bound up to logarithmic factors, which suggests the optimality of our algorithm. While our lower bound is proved for Markov games with linear mixture structure, we argue that it is also a valid lower bound for Markov games with linear structure \citep{xie2020learning}.
\end{itemize}

\noindent\textbf{Notation} We use lower case letters to denote scalars,  lower and upper case bold letters to denote vectors and matrices. We use $\| \cdot \|$ to indicate Euclidean norm, and for a semi-positive definite matrix $\bSigma$ and any vector $\xb$, $\| \xb \|_{\bSigma} := \| \bSigma^{1/2} \xb \| = \sqrt{\xb^{\top} \bSigma \xb}$. For a real value $x$ and an interval $[a,b]$, we use $[x]_{[a,b]}$ to indicate the projection of $x$ onto $[a,b]$.
We also use the standard $O$ and $\Omega$ notations. We say $a_n = O(b_n)$ if and only if $\exists C > 0, N > 0, \forall n > N, a_n \le C b_n$; $a_n = \Omega(b_n)$ if and only if $\exists C > 0, N > 0, \forall n > N, a_n \ge C b_n$. The notation $\tilde{O}$ is used to hide logarithmic factors. 

%It is worth noting that turn-based games \citep{jin2019learning,xie2020learning} where at each state the reward and transition kernel only depend on the action of one of the players, can be viewed as a special case of simultaneous games analyzed in this paper, our algorithms and guarantees are readily applicable to the turn-based setting.

\section{Related Work}
%In this section, we review the prior works that related to our work.

\textbf{Tabular Markov game.}
Under the tabular setting,
\citet{littman1996generalized} extended the value iteration and Q-learning algorithms \citep{watkins1989learning} to zero-sum Markov games.
\citet{littman2001friend,greenwald2003correlated,hu2003nash} further extended it to general-sum Markov games with $n$-player. \citet{hansen2013strategy} provided the first strong polynomial algorithm for solving two-player turn-based Markov games.
 %and that in \citet{perolat2018actor,srinivasan2018actor}
%extends the actor-critic algorithm \citep{konda2000actor}. 
\citet{sidford2018near} proposed a variance-reduced variant of the minimax Q-learning algorithm with near-optimal sample complexity. %However, all these results require a generative model, and they are limited to turn-based games, a special case of simultaneous-move games. 
\citet{lagoudakis2012value, perolat2015approximate, fan2020theoretical} considered value-iteration with function approximation and established finite-time convergence to the NEs of two-player zero-sum Markov games. Their results are based on the framework of fitted value-iteration 
\citep{munos2008finite}.  \citet{jia2019feature} studied turn-based zero-sum Markov games, where the transition model is assumed to be embedded in some $d$-dimensional feature space.  \citet{cui2020minimax} proposed an algorithm for turn-based zero-sum Markov games based on plug-in estimator and achieved minimax sample complexity. For the simultaneous-move zero-sum Markov games, \citet{zhang2020model} proposed an algorithm which achieved minimax sample complexity if the algorithm is reward-agnostic. 
All the above works either assume a generative oracle or a well explored behavioral policy for drawing transitions, therefore bypassing the exploration issue. \citet{bai2020provable} proposed a VI-ULCB algorithm for tabular episodic zero-sum Markov games, which achieves $\tilde O(\sqrt{H^{3}S^{2}ABT})$ regret for simultaneous move (i.e., general Markov game) and $O(\sqrt{H^{3}S^{2}(A+B)T})$ regret for turn-based game, where $A$ and $B$ are the number of actions for each player, $H$ is the length of the game, and $T$ is the total number of steps played in the game. They also proved an $\Omega(\sqrt{H^{2}S(A+B)T})$ lower bound. %\todoq{Is bai's setting the same as Wei? Also, I remember there are two papers by Bai and Jin. What about the other one?} 
For general Markov game, \citet{bai2020near} proposed an Optimistic Nash Q-learning algorithm with a regret of $\tilde O(\sqrt{H^{4}SABT})$, and an Optimistic Nash V-learning algorithm with a regret of $\tilde O(\sqrt{H^{5}S(A+B)T})$, both of which improve the regret in \citet{bai2020provable} in the dependence on $S,A,B$.
% One of the best known algorithm is VI-ULCB, the regret to find an approximate Nash equilibrium in $\tilde O(\sqrt{H^{3}S^{2}ABT})$\citep{bai2020provable}. 
The best known regret is achieved by Nash-VI proposed in \citet{liu2020sharp}, which is $\tilde O(\sqrt{H^{2}SABT})$. As can be seen, without assuming the access to a generative model or a well explored behavioral policy, there is still a gap between the upper and lower regret bounds for existing algorithms, even for the simplest tabular Markov games.
%Our work %does not require the access to the generative model or a well explored behavioral policy, and 
%is also directly dealing with the exploration-exploitation tradeoff, but needs to tackle the additional challenge of linear function approximation.

%Our work builds on a line of research on provably efficient methods for MDPs without additional assumptions on the sampling model. Most of the existing work focus on the tabular setting; see e.g., \citet{strehl2006pac, jaksch2010near, osband2016generalization, azar2017minimax, dann2017unifying, agrawal2017optimistic,jin2018q, russo2019worst, rosenberg2019online, jin2019learning, zanette2019tighter, simchowitz2019non, dong2019q} and the references therein. 

\noindent\textbf{Online RL with linear function approximation.}
There are several lines of work aiming at providing theoretical guarantees for online RL with function approximation. The first line of work focus on the linear function approximation setting, which assumes that the MDP (e.g., transition probability, reward, or value function) can be represented as a linear function of some given feature mapping. These works proposed algorithms which enjoy sample complexity/regret scaling with the dimension of the feature mapping, rather than the cardinality of state and action spaces. For example, \citet{yang2019sample,jin2019provably,wang2019optimism, zanette2020frequentist, he2020logarithmic} considered the linear MDP model, where the transition probability function and reward function are linear in some feature mapping over state-action pairs. \citet{zanette2020learning} studied MDPs with low inherent Bellman error, where the value functions are nearly linear w.r.t. the feature mapping. \citet{yang2019reinforcement, modi2019sample, jia2020model, ayoub2020model, cai2019provably, zhou2020provably, he2020logarithmic} studied the linear mixture MDPs, where the transition probability kernel is a linear mixture of a number of basis kernels. Inspired by linear mixture MDPs, we introduce the linear mixture Markov game. 

\section{Preliminaries}
In this section, we introduce the setup of the episodic two-player zero-sum Markov games with simultaneous moves and the linear mixture structure we use in this paper. 

\subsection{Two-Player Markov Games}
The two-player zero-sum Markov game (MG) \citep{shapley1953stochastic,littman1994markov} is a generalization of the standard Markov decision process (MDP) where the max-player seeks to maximize the total return, and the min-player seeks to minimize the total return.

\noindent\textbf{Simultaneous-move MG.} Formally, we denote a two-player zero-sum simultaneous-moves episodic Markov Game by a tuple $M(\cS, \cA_{\max}, \cA_{\min}, H, \{r_{h}\}_{h=1}^{H}, \{\mathbb{P}_{h}\}_{h=1}^{H})$. $\cS$ is a countable state space, $\cA_{\max}, \cA_{\min}$ are the finite action spaces of the max-player and  the min-player respectively. $H$ is the length of the game/episode. For simplicity, we assume the reward function for the max-player $\{r_{h}\}_{h=1}^{H}$ is deterministic and known function $r_{h}: \cS \times \cA_{\max} \times \cA_{\min} \rightarrow [-1,1]$. $\mathbb{P}_{h}(s'|s, a, b)$ is the transition probability function which denotes the probability for state $s$ to transit to state $s'$ given players' action pair $(a,b)$ at step $h$. 

\noindent\textbf{Markov Policy and Value Function.} We first define the stochastic policies, which give distributions over the actions. A policy $\pi = \{\pi_{h} : \cS \rightarrow \Delta_{\cA_{\max}}\}_{h=1}^{H}$ is a collection of functions which map a state $s \in \cS$ to a distribution of actions. Here $\Delta_{\cA_{\max}}$ is the probability simplex over action set $\cA_{\max}$. Similarly, we can define a policy $\nu = \{\nu_{h} : \cS \rightarrow \Delta_{\cA_{\min}}\}_{h=1}^{H}$ for the min-player, where $\Delta_{\cA_{\min}}$ is the probability simplex over action set $\cA_{\min}$. %is a collection function which map a state $s \in \cS$ to a distribution of actions. 
We use the notation $\pi_{h}(a|s)$ and $\nu_{h}(b|s)$ to present the probability of taking action $a$ or $b$ for state $s$ at step $h$ under Markov policy $\pi,\nu$ respectively. We define the action-value function (a.k.a., $Q$ function) $Q_{h}^{\pi, \nu}:\cS\times\cA_{\max}\times\cA_{\min}\rightarrow \mathbb{R}$ as follows
\begin{align*}
Q_{h}^{\pi, \nu}(s,a,b) &= \EE_{\pi,\nu,h,s,a,b}\bigg[\sum_{h'=h}^{H}r(s_{h'}, a_{h'}, b_{h'})\bigg| s_{h}=s, a_{h}=a, b_{h}=b\bigg],
\end{align*}
and the value function $V_{h}^{\pi, \nu}:\cS\rightarrow \mathbb{R}$ as follows
\begin{align*}
V_{h}^{\pi, \nu}(s) &= \EE_{a\sim \pi_{h}(\cdot|s),b\sim \nu_{h}(\cdot|s)} Q_{h}^{\pi, \nu}(s, a, b), \qquad V_{H+1}^{\pi, \nu}(s) = 0.
\end{align*}
In the definition of $Q_{h}^{\pi, \nu}$, $\EE_{\pi,\nu,h,s,a,b}$ is an expectation over state-action pairs of length $H - h + 1$ induced by the policy $(\pi, \nu)$ and the transition probability of the MG $M$, when initializing the process with the triplet $(s,a,b)$ at step $h$. Because $r_{h}(\cdot,\cdot,\cdot) \in [-1,1]$, it is easy see that both $Q$ functions and value functions are bounded
\begin{align*}
|Q_{h}^{\pi, \nu}(\cdot,\cdot,\cdot)|\leq H, \qquad |V_{h}^{\pi, \nu}(\cdot)|\leq H.
\end{align*}
Furthermore, for any joint distribution $\sigma \in \Delta(\cA_{\text{max}}\times \cA_{\text{min}})$, we denote by $\cP_{\text{max}}\sigma$ the marginal distribution for the max-player and by $\cP_{\text{min}}\sigma$ the marginal distribution for the min-player. 

\noindent\textbf{Best Response and Bellman Equation.} The goal of the max-player is to maximize the total rewards. The goal of the min-player is to minimize the total rewards that the max-player will get because this is a zero-sum game. In other words,  the max-player wants to maximize $V_{h}^{\pi, \nu}(\cdot)$ by choosing a good policy $\pi$,  while the min-player wants to minimize $V_{h}^{\pi, \nu}(\cdot)$ by choose a good policy $\nu$. Accordingly, we can define the action-value function and the value function when the max-player gives the best response to a fixed policy $\nu$ of the min-player:
\begin{align*}
Q_{h}^{*, \nu}(s,a,b) &= \max_{\pi}Q_{h}^{\pi, \nu}(s,a,b),\qquad
V_{h}^{*, \nu}(s) = \max_{\pi}V_{h}^{\pi,\nu}(s). 
\end{align*}
By symmetry, we can also define 
\begin{align*}
Q_{h}^{\pi, *}(s,a,b) &= \min_{\nu}Q_{h}^{\pi, \nu}(s,a,b), \qquad
V_{h}^{\pi, *}(s) = \min_{\nu}V_{h}^{\pi,\nu}(s).
\end{align*}
For any function  $V: \cS \rightarrow \mathbb{R}$, we introduce the shorthands:
\begin{align*}
[\mathbb{P}_{h}V](s, a, b) &= \EE_{s' \sim \mathbb{P}_{h}(\cdot|s,a,b)}V(s'), \qquad  
[\mathbb{V}_{h}V](s, a, b) = [\mathbb{P}_{h}V^2](s, a, b) - \big([\mathbb{P}_{h}V](s, a, b)\big)^2,
\end{align*}
where $V^{2}$ stands for the function whose value at $s$ is $V^{2}(s)$. Using these notation, we have following Bellman equations:
\begin{align*}
Q_{h}^{\pi, \nu}(s, a, b) &= r(s ,a, b) + [\PP_{h}V_{h+1}^{\pi, \nu}](s, a, b),
\end{align*}
and the Bellman optimality equation \citep{shapley1953stochastic}:
\begin{align*}
Q_{h}^{\pi,*}(s, a, b) &= r(s ,a, b) + [\PP_{h}V_{h+1}^{\pi,*}](s, a, b), \qquad
V_{h}^{\pi,*}(s) = \inf_{\sigma \in \Delta_{\min}}\EE_{a\sim \pi_{h}(\cdot|s),b\sim \sigma}Q_{h}^{\pi,\nu}(s, a, b).
% Q_{h}^{*,\nu}(s, a, b) &= r(s ,a, b) + [\PP_{h}V_{h+1}^{*,\nu}](s, a, b).    
\end{align*}

\noindent\textbf{Nash Equilibrium.} A Nash Equilibrium (NE) of the game %\todoq{starting at state $s \in \cS$}
is a pair of policies $\pi^{*},\nu^{*}$ such that
\begin{align}
V_{1}^{\pi^{*},\nu^{*}}(s) =   V_{1}^{\pi^{*}, *}(s) = V_{1}^{*, \nu^{*}}(s), \text{ for all }s\in \cS.\label{eq:NEdef}
\end{align}
\eqref{eq:NEdef} means that $(\pi^{*},\nu^{*})$ are the best response to each other, so no player can do better by only changing her own policy. Nash equilibrium can also be viewed as ``the best response to the best response''. For most applications, they are the ultimate solutions we want to pursue. We further abbreviate the value of the Nash equilibrium $V_{1}^{\pi^{*},\nu^{*}}(s)$ as  $V_{1}^{*}(s)$. This is because the value of the Nash equilibrium is irrelevant to the choice of $(\pi^{*},\nu^{*})$ which is a direct corollary of the following weak duality property:
\begin{proposition}[Weak Duality,  \citealt{xie2020learning}]\label{prop:weak} Given the NE $(\pi^{*}, \nu^{*})$ of a game, for any policy pair $(\pi, \nu)$ we have that 
\begin{align}
V_{1}^{*, \nu}(s) \geq V_{1}^{\pi^{*}, \nu^{*}}(s) \geq V_{1}^{\pi, *}(s), \text{ for all } s\in \cS.\label{eq:Weak Duality}
\end{align}
\end{proposition}
% \begin{proof}
% We prove the first inequality of \eqref{eq:Weak Duality} as follows,
% \begin{align*}
% V_{1}^{*, \nu}(s) \geq V_{1}^{\pi^{*}, \nu}(s) \geq V_{1}^{\pi^{*}, \nu^{*}}(s),    
% \end{align*}
% where the first inequality is by the best response of $\nu$, the second inequality is by the best response of $\nu^{*}$.
% Similarly, we prove the second inequality of \eqref{eq:Weak Duality} as follows, 
% \begin{align*}
% V_{1}^{\pi^{*}, \nu^{*}}(s)\geq  V_{1}^{\pi, \nu^{*}}(s)\geq V_{1}^{\pi, *}(s),        
% \end{align*}
% where the first inequality is by the best response of $\pi^{*}$, the second inequality is by the best response of $\nu$.
% \end{proof}
%\todoq{why do we need to prove this?}

\noindent\textbf{Learning Objective.} 
% In the offline setting, there is a central controller that controls both max-player and min-player. The controller doesn't know the kernels $\{\mathbb{P}_{h}\}_{h=1}^{H}$ but knows the reward functions $\{r_{h}\}_{h=1}^{H}$. By interacting with the environment, the controller aims to find the Nash equilibrium of the game. For each $k \geq 1$, at the beginning of the $k$-th episode, the environment
% picks the initial state $s_{1}^{k}$. Central player selects a correlated policy $\mu = \{\mu_{h}\} : \cS \rightarrow \Delta_{\cA_{\max}\times\cA_{\min}}$ to be followed in this episode. As the agents follow the policy through the episode, it takes the action $(a_{h}^{k}, b_{h}^{k})\sim \mu_{h}^{k}(\cdot,\cdot|s_{h}^{k})$ and observes the sequence of states $\{s_{h}^{k}\}_{h=1}^{H}$ with $s_{h+1}^{k} \sim \mathbb{P}_{h}(\cdot|s_{h}^{k},a_{h}^{k},b_{h}^{k})$. For the max-player herself, she records the marginal distribution $\pi_{h}^{k}(\cdot|s) = \cP_{\max}\mu_{h}^{k}(\cdot,\cdot|s)$ as the policy $\pi^{k}$ she has learned in this $k$-th episode. So as the min-player record $\nu_{h}^{k}(\cdot|s) = \cP_{\min}\mu_{h}^{k}(\cdot,\cdot|s)$ as the policy $\nu^{k}$ she has learned in this k-th episode.
The weak duality in Proposition \ref{prop:weak} suggests that the NE value $V_{1}^{*}(s)$ is sandwiched between
$V_{1}^{*, \nu}(s)$ and $V_{1}^{\pi,*}(s)$. So it is natural to measure the suboptimality of learned policies $(\pi^{k}, \nu^{k})$ at the $k$-th episode by the gap between their performance and the performance of the optimal strategy (i.e., Nash equilibrium) when playing against the best responses respectively:
\begin{align*}
&V_{1}^{*, \nu^{k}}(s) - V_{1}^{\pi^{k},*}(s) = \big[V_{1}^{*, \nu^{k}}(s) - V_{1}^{*}(s)\big] + [V_{1}^{*}(s)-V_{1}^{\pi^{k},*}(s)].
\end{align*}
% \begin{definition}[$\epsilon$-approximation Nash equilibrium \citep{bai2020near}].
% A pair of polices $(\hat{\pi}, \hat{\nu})$ is called an $\epsilon$-approximate Nash equilibrium of the game starting at the state $s_{1}$ if the following inequality holds
% \begin{align*}
% V_{1}^{*, \hat{\nu}}(s_{1}) - V_{1}^{ \hat{\pi},*}(s_{1}) \leq \epsilon.     
% \end{align*}
% \end{definition}
% The goal of the controller is to find the Nash equilibrium of the game in episodic setting. To be specific we want to find $\pi^{*}, \nu^{*}$ such that 
% \begin{align*}
% V_{1}^{*, \nu^{*}}(s) = V_{1}^{\pi^{*}, \nu^{*}}(s) = V_{1}^{\pi^{*}, *}(s). 
% \end{align*}
% Therefore $V_{1}^{*, \hat{\nu}}(s) - V_{1}^{ \hat{\pi},*}(s)$ can represent how well $(\pi^{k},\nu^{k})$ approximate the NE in the k-th episode. 
Accordingly, we aim to design a learning algorithm that outputs a sequence $\{\pi^{k},\nu^{k}\}_{k}$ based on past information, and minimize the regret over first $K$ episodes defined as follows:
\begin{align*}
\text{Regret}(M, K) &= \sum_{k=1}^{K} \Big[ V_{1}^{*, \nu^{k}}(s_{1}^{k})  -  V_{1}^{\pi^{k},*}(s_{1}^{k})\Big].
\end{align*}
This measure has been widely used in previous work that studies the offline learning of two-player game \citep{bai2020near,xie2020learning,liu2020sharp}. Following \citet{bai2020near,xie2020learning,liu2020sharp}, we assume the central controller can choose a joint distribution $\mu^k$ for both the max-player and min-player in each episode as their policies, and we set $\pi^k = \cP_{\text{max}}\mu^k$ and $\nu^k = \cP_{\text{min}}\mu^k$ automatically. In this paper, we focus on proving high probability bounds on the regret $\text{Regret}(M, K)$, as well as lower bounds in expectation.

\noindent\textbf{Episodic Linear Mixture Markov Games.}
In this work, we consider a class of MGs called \emph{linear mixture MGs}, inspired by the linear mixture/kernel MDPs studied in \cite{modi2019sample,jia2020model, ayoub2020model} for the single-agent RL. Linear mixture MGs assume that at each step $h$, the transition probability function $\PP_h(s'|s,a,b)$ is a linear combination of $d$ feature mappings $\phi_i(s'|s,a,b)$, i.e.,
\begin{align}
    \PP_h(s'|s,a,b) = \sum_{i=1}^d \theta_{i,h}\phi_i(s'|s,a,b),\notag
\end{align}
where each feature mapping $\phi_i(s'|s,a,b)$ is a function defined on the state-action-action-state pair $(s,a,b,s') \in \cS \times \cA_{\text{max}} \times \cA_{\text{max}} \times \cS$. For the sake of simplicity, we use a vector function $\bphi = [\phi_1,\cdots, \phi_d] \in \RR^d$ to denote the collection of $\phi_i$. After proper normalization, we assume $\bphi$ satisfy that for any bounded function $V:\cS \rightarrow [-1,1]$ and any tuple $(s,a,b) \in \cS \times \cA_{\max} \times \cA_{\min}$, we have 
\begin{align}
\|\bphi_{V}(s,a, b)\|_{2} \leq 1, \label{mapping}
\end{align}
where $\bphi_{V}(s,a,b) = \sum_{s'\in \cS}\bphi(s'|s,a,b)V(s')$. Formally, we define linear mixture MGs as follows:
\begin{definition}
$M(\cS, \cA_{\max}, \cA_{\min}, H,  \{r_{h}\}_{h=1}^{H}, \{\mathbb{P}_{h}\}_{h=1}^{H})$ is called a time inhomogeneous, episodic $B$-bounded linear mixture MG if there exist $H$ unknown vectors $\btheta_{h} \in \mathbb{R}^{d}$ satisfying for any $h \in [H]$, $\|\btheta_{h}\|_{2} \leq B$, and a known feature mapping $\bphi$ satisfying \eqref{mapping}, such that $\mathbb{P}_{h}(s'|s,a,b) = \langle \bphi(s'|s,a,b), \btheta_{h}\rangle$ for any state-action-action-state triplet $(s, a, b, s')$ and any step $h$. We denote the linear mixture MG by $M_{\btheta}$ for simplicity. 
\end{definition}
In this paper, we assume the underlying linear mixture MG is parameterized by $\{\btheta^{*}_{h}\}_{h=1}^{H}$, denoted by $M_{\btheta^*}$.

\paragraph{Difference between linear and linear mixture MGs.} Linear mixture MGs assume that at each step $h$, the transition probability function $\PP_h(s'|s,a,b)$ is a linear combination of $d$ feature mappings $\bphi_i(s'|s,a,b)$ for $i=1, \ldots,d$, i.e., $\PP_h(s'|s,a,b) =\langle \bphi(s'|s,a,b), \btheta_{h}\rangle$. The linear MG setting considered by \citet{xie2020learning}, however, assumes $\PP_h(s'|s,a,b) =\langle \bphi(s,a,b), \bmu_{h}(s')\rangle$, where $\bmu_{h}(\cdot)$ is an unknown vector-valued measure function on $S$. These two models are different and do not include each other in general. For instance, consider the following MG which is inspired by \citet{zhou2020provably}: $\cS = \mathbb{Z}$, $\cA_{\text{max}} =\cA_{\text{min}}= \mathbb{N}$ and $\PP_h(s'|s,a,b) = \sum_{i=1}^d \theta_i^h p_i(s'|s,a,b)$, $p_i(s'|s,a,b) = \ind(s'>=s) (a+b)^{s'-s}\exp(-(a+b))/(s'-s)!$. This MG is a linear mixture MG but not a linear MG.

% \CC{We assume the underlying parameter $\theta_h^{*}$’s are unknown for linear mixture MGs, while other quantities, as well as feature mappings, are known.}

\section{Algorithm}

In this section, we propose our algorithm $\algname$ in Algorithm \ref{alg:Offline}. Due to the space limit, we only show the detailed update rules for the max-player in Algorithm \ref{alg:Offline}, and the full algorithm is presented in Algorithm \ref{alg:fullalg} in Appendix~\ref{sec:Full Algorithm}. All the parameters corresponding to the max-player are marked by an overline, while the parameters for the min-player are marked by an underline. 
%\todoq{can we explain the bar and lower bar here?}
%\iffalse

\begin{algorithm*}[ht]
\caption{$\algname$ }\label{alg:Offline}
\begin{algorithmic}[1]
\STATE \textbf{Input:} Regularization parameter $\lambda$, number of episode $K$, number of horizon $H$,  approximation error $\epsilon$.
% \STATE receives $s_{1}$.
\STATE For any $h$, $\overline{\bSigma}^{(i)}_{1,h}  \leftarrow \underline{\bSigma}^{(i)}_{1,h} \leftarrow \lambda \Ib$; $\overline{\bbb}^{(i)}_{1,h} \leftarrow \underline{\bbb}^{(i)}_{1,h} \leftarrow \zero$; $\overline{\btheta}^{(i)}_{1,h} \leftarrow \underline{\btheta}^{(i)}_{1,h} \leftarrow  \zero$, for $i \in \{0,1\}$.  
\FOR{$k = 1,\ldots, K$}
\STATE $\overline{V}_{k,H+1}(\cdot) \leftarrow 0$, $\underline{V}_{k,H+1}(\cdot) \leftarrow 0$.
\FOR{$h = H, \ldots, 1$}
\STATE Set $\overline{Q}_{k,h}(\cdot,\cdot,\cdot)$ as in \eqref{eq:overq}, and $\underline{Q}_{k,h}(\cdot,\cdot,\cdot)$ in a similar way (See Algorithm \ref{alg:fullalg}).
\FOR {$s\in \cS$}
\STATE Let $\mu_{h}^{k}(\cdot,\cdot|s) = \epsilon\text{-CCE}(\overline{Q}_{k,h}(s,\cdot,\cdot),\underline{Q}_{k,h}(s,\cdot,\cdot))$.\alglinelabel{alg:CCE}
\STATE $\overline{V}_{k,h}(s) = \EE_{(a,b) \sim \mu_{h}^{k}(\cdot,\cdot|s)}\overline{Q}_{k,h}(s, a, b)$, $\underline{V}_{k,h}(s) = \EE_{(a,b) \sim \mu_{h}^{k}(\cdot,\cdot|s)}\underline{Q}_{k,h}(s, a, b)$\alglinelabel{alg:defv}.
\STATE $\pi_{h}^{k}(\cdot|s) = \cP_{\max}\mu_{h}^{k}(\cdot,\cdot|s)$, $\nu_{h}^{k}(\cdot|s) = \cP_{\min}\mu_{h}^{k}(\cdot,\cdot|s)$.
\ENDFOR
\ENDFOR
\STATE Receives $s_{1}^{k}$
\FOR{$h = 1, \ldots, H$}
\STATE Take action $(a_{h}^{k}, b_{h}^{k}) \sim \mu_{h}^{k}(\cdot,\cdot|s_{h}^{k})$ and central controller receives $s_{h+1}^{k} \sim \mathbb{P}(\cdot | s_{h}^{k}, a_{h}^{k}, b_{h}^{k})$.
\STATE Set $\mathbb{V}^{\text{est}}\overline{V}_{k,h+1}(s_{h}^{k},a_{h}^{k},b_{h}^{k})$ as in \eqref{eq:Vestimator} and $\overline{E}_{k, h}$ as in \eqref{eq:Eesptimator},  $\overline{\sigma}_{k,h}$ as in \eqref{eq:smallsigma}.
\STATE Set  $\overline{\bSigma}_{k+1,h}^{(0)},  \overline{\bbb}_{k+1,h}^{(0)}$ as in \eqref{eq:BigSigma0} and \eqref{eq:bvector0}, $\overline{\bSigma}_{k+1,h}^{(1)}, \overline{\bbb}_{k+1,h}^{(1)}$ as in \eqref{eq:BigSigma1} and \eqref{eq:bvector1}.
\STATE Set $ \underline{\bSigma}_{k+1,h}^{(0)}, \underline{\bbb}_{k+1,h}^{(0)}, \underline{\bSigma}_{k+1,h}^{(1)}, \underline{\bbb}_{k+1,h}^{(1)}, \mathbb{V}^{\text{est}}\underline{V}_{k,h+1}(s_{h}^{k},a_{h}^{k},b_{h}^{k}), \underline{E}_{k, h}, \underline{\sigma}_{k,h}$ in similar ways (See Algorithm \ref{alg:fullalg}).
\STATE Set $\overline{\btheta}_{k+1,h}^{(i)} \leftarrow \big[\overline{\bSigma}_{k+1,h}^{(i)}\big]^{-1}\overline{\bbb}_{k+1,h}^{(i)} $,  $\underline{\btheta}_{k+1,h}^{(i)} \leftarrow \big[\underline{\bSigma}_{k+1,h}^{(i)}\big]^{-1}\underline{\bbb}_{k+1,h}^{(i)} $, $i=0,1$
\ENDFOR
\ENDFOR
\end{algorithmic}
\end{algorithm*}

%\fi

%\subsection{Proposed Algorithm}

 To achieve the near-minimax optimality of solving a linear mixture MG, $\algname$ adopts the following three techniques, which we will introduce in sequence. 

\noindent\textbf{Value-targeted regression}
To find the NE of an MG, it suffices to find good estimates of the optimal value functions $\vvalue_h^{*, \nu^k}$ and $\vvalue_h^{\pi^k, *}$. By the Bellman optimality equations and the definition of linear mixture MGs, it is sufficient to estimate the underlying unknown parameter $\btheta_h^*$ up to good accuracy. Inspired by the UCRL with ``value-targeted regression'' (VTR) proposed by \citet{jia2020model, ayoub2020model}, $\algname$ uses a supervised learning framework to learn $\btheta_h^*$. In the sequel, we introduce how the VTR framework works at episode $k$ and step $h$. At the beginning of episode $k$, $\algname$ maintains two estimated value functions: optimistic value function $\overline{\vvalue}_{k, h+1}$ for the max-player, which overestimates the optimal value function $\vvalue_h^{*, \nu^k}$, and optimistic value function $\underline{\vvalue}_{k, h+1}$ for the min-player, which underestimates the value function $\vvalue_h^{\pi^k, *}$. We focus on the overestimate $\overline{\vvalue}_{k, h+1}$ first. Note that the following equation holds due to the definition of linear mixture MGs:
\begin{align}
    &[\PP_h\overline{\vvalue}_{k, h+1}](s_h^k, a_h^k, b_h^k) = \la \bphi_{\overline{\vvalue}_{k, h+1}}(s_h^k, a_h^k, b_h^k), \btheta_h^*\ra,\label{eq:expect}
\end{align}
which suggests that $(\bphi_{\overline{\vvalue}_{k, h+1}}(s_h^k, a_h^k, b_h^k), \overline{\vvalue}_{k, h+1}(s_{h+1}^k))$ can be regarded as a context and the corresponding targeted value of a linear regression problem with the unknown parameter $\btheta_h^*$.

Therefore, $\algname$ constructs $\overline{\btheta}_{k,h}^{(0)}$ as the estimator of $\btheta_h^*$ based on linear regression on $(\bphi_{\overline{\vvalue}_{k, h+1}}(s_h^k, a_h^k, b_h^k), \overline{\vvalue}_{k, h+1}(s_{h+1}^k))$ (the detailed construction of $\overline{\btheta}_{k,h}^{(0)}$ will be specified later). Due to the randomness of $s_{h+1}^k$, $\overline{\btheta}_{k,h}^{(0)}$ can not estimate $\btheta_h^*$ exactly. Therefore $\algname$ also constructs an ellipsoid $\overline{\cC}_{k,h}^{(0)}$ centered at $\overline{\btheta}_{k,h}^{(0)}$ as the confidence set, which contains $\btheta_h^*$ with high probability:
\begin{align}
\overline{\cC}^{(0)}_{k,h} &:=  \bigg\{\btheta: \bigg\|\Big[\overline{\bSigma}^{(0)}_{k, h}\big]^{1/2}(\btheta - \overline{\btheta}^{(0)}_{k,h})\Big\|_{2} \leq \beta_{k}^{(0)}\bigg\}.\label{eq:Upepp}    
\end{align}
Here $\overline{\bSigma}^{(0)}_{k, h}$ is the ``covariance matrix'' of the context $\bphi_{\overline{\vvalue}_{k, h+1}}(s_h^k, a_h^k, b_h^k)$, and $\beta_{k}^{(0)}$ is the radius of the confidence set. Both of them will be specified later. Then, to encourage the agent to explore, $\algname$ constructs an optimistic action-value function $\overline{Q}_{k,h}$ as follows, following the ``optimism-in-the-face-of-uncertainty'' principle \citep{abbasi2011improved}:
\begin{align}
\overline{Q}_{k,h} &:= \Big[r_{h} + \max_{\btheta \in \overline{\cC}^{(0)}_{k,h}}\langle \btheta_{k,h}, \bphi_{\overline{V}_{k,h+1}} \rangle\Big]_{[-H, H]},\label{eq:overq0}
\end{align}
where the projection onto $[-H,H]$ is because the action-value function of the Markov game lies in $[-H, H]$.
The closed-form solution of \eqref{eq:overq0} is as follows
\begin{align}
\overline{Q}_{k,h}(\cdot,\cdot,\cdot) = \Big[ r_{h}(\cdot,\cdot,\cdot) + \langle \overline{\btheta}^{(0)}_{k,h}, \bphi_{\overline{V}_{k,h+1}}(\cdot,\cdot,\cdot) \rangle + \beta^{(0)}_{k}\Big\|\big[\overline{\bSigma}_{k,h}^{(0)}\big]^{-1/2}\bphi_{\overline{V}_{k,h+1}}(\cdot,\cdot,\cdot)\Big\|_{2}\Big]_{[-H,H]}, \label{eq:overq}    
\end{align}

Similar procedures can be applied to construct the confidence set $\underline{\cC}^{(0)}_{k,h}$ and the optimistic action-value function $\underline{Q}_{k,h}$ for the min-player with parameters $\underline{\btheta}^{(0)}_{k,h}$, $\underline{\bSigma}^{(0)}_{k, h}$. Finally, $\algname$ constructs the optimistic value functions $\overline{V}_{k,h}, \underline{V}_{k,h}$ and the policy $\mu_h^k$ based on $\overline{Q}_{k,h}, \underline{Q}_{k,h}$ for the current episode and step (which will be specified later).

\noindent\textbf{Coarse Correlated Equilibrium (CCE).}
Now we introduce how to compute the $(\pi_h^k, \nu_h^k)$ based on the optimistic action-value functions $\overline{Q}_{k,h}, \underline{Q}_{k,h}$. Unlike the single-agent RL, we cannot certify the policy by independently solving max-min problem on $\overline{Q}$ or $\underline{Q}$. This is  because $\overline{Q}$ and $\underline{Q}$ are not the estimators of action-value function for the NE but the estimators of action-value function for the best response. Thus we must coordinate both players for their choices of actions. After we get  $\overline{Q}_{k,h}(s,\cdot,\cdot)$ for the max-player and $\underline{Q}_{k,h}(s,\cdot,\cdot)$ for the min-player, we solve
a general-sum matrix game to find the Coarse Correlated Equilibrium (CCE), following \cite{xie2020learning}. Here we give the formal definition of CCE as follows:
\begin{definition}[\citealt{moulin1978strategically, aumann1987correlated}]
Given two payoff matrices $Q_{\text{max}}, Q_{\text{min}} \in \RR^{|\cA_{\text{max}}|\cdot |\cA_{\text{min}}|}$, we denote the $\epsilon$-Coarse Correlated Equilibrium ($\epsilon$-CCE) as a joint distribution $\sigma$ over $\cA_{\text{max}}$ and $\cA_{\text{min}}$ satisfying that
\begin{align}
    \EE_{(a,b) \sim \sigma} Q_{\text{max}}(a,b) \geq \max_{a' \in \cA_{\text{max}}}\EE_{b \sim \cP_{\text{min}}\sigma }Q_{\text{max}}(a',b)-\epsilon,\notag \\
    \EE_{(a,b) \sim \sigma} Q_{\text{min}}(a,b) \leq \min_{b' \in \cA_{\text{min}}}\EE_{a \sim \cP_{\text{max}}\sigma }Q_{\text{min}}(a,b')+\epsilon.\notag
\end{align}
\end{definition}
$\algname$ computes the distribution $\mu_h^k(\cdot, \cdot|s)$, a $\epsilon$-CCE of $\overline{Q}_{k,h}, \underline{Q}_{k,h}$ for each state $s$ in Line \ref{alg:CCE}. Then $\algname$ selects the value functions $\overline{V}_{k,h}, \underline{V}_{k,h}$ as the expectation of $\overline{Q}_{k,h}, \underline{Q}_{k,h}$ over the policies $\mu_{h}^{k}$ as in Line \ref{alg:defv} of Algorithm \ref{alg:Offline}. The difference between CCE and NE is whether the policy of each player is independent of each other. The policy $\mu_{h}^k$ given by $\epsilon$-CCE is correlated for each player because it is found in the class $\Delta_{\cA_{\max}\times \cA_{\min}}$ rather than $\Delta_{\cA_{\max}}\times \Delta_{\cA_{\min}}$. After obtaining $\mu^k$, $\algname$ sets $\pi_{h}^{k}(\cdot|s) = \cP_{\max}\mu_{h}^{k}(\cdot,\cdot|s)$ and $\nu_{h}^{k}(\cdot|s) = \cP_{\min}\mu_{h}^{k}(\cdot,\cdot|s)$, i.e., the marginal distributions of $\mu_{h}^k$. Notice that a Nash equilibrium always exists and 
a Nash equilibrium for a general-sum game is also a CCE. Thus a CCE always exists, so does $\epsilon$-CCE. 
\begin{remark}
Since we assume the action spaces are finite, the constraints for $\epsilon$-CCE can be rewritten as $|\cA_{\max}| + |\cA_{\min}|$ linear constraints, which can be efficiently solved by linear programming (See e.g., \citet{bai2020near,liu2020sharp}).
\end{remark}

\noindent\textbf{Weighted linear regression for value function estimation.}
Now we specify how to %choose the confidence sets $\overline{\cC}_{k,h}^{(0)}, \underline{\cC}_{k,h}^{(0)}$, which is based on the choice of the covariance matrix $\overline{\bSigma}_{k,h}^{(0)}, \underline{\bSigma}_{k,h}^{(0)}$ and 
construct the estimators $\overline{\btheta}_{k,h}^{(0)}, \underline{\btheta}_{k,h}^{(0)}$. For the simplicity, we only show the construction for the max-player and the construction for the min-player is presented in Appendix~\ref{sec:Full Algorithm}. With the linear structure of $\overline{\vvalue}_{k, h+1}(s_{h+1}^k)$ in \eqref{eq:expect}, it is natural to set the estimator $\overline{\btheta}_{k,h}$ as the minimizer to the linear regression problem with square loss over context-target pairs $(\bphi_{\overline{\vvalue}_{k, h+1}}(s_h^k, a_h^k, b_h^k), \overline{\vvalue}_{k, h+1}(s_{h+1}^k))$, similar to UCRL \citep{jia2020model, ayoub2020model}. %and OMVI by \citet{xie2020learning}.
% For a value function $\overline{V}_{k,h+1}$ constructed based on data received before episode $k$
% and the state action pair $(s_h^k,a_h^k, b_h^k)$ visited in stage $h$ of episode $k$, we have the following equation
% \begin{align*}
%     [\PP_h\vvalue_{k, h+1}](s_h^k, a_h^k,b_h^k) &=\bigg\la\sum_{s'}\bphi(s'|s_h^k, a_h^k,b_h^k)\overline{V}_{k, h+1}(s'), \btheta_h^*\bigg\ra\\
%     &= \big\la\bphi_{\overline{V}_{k, h+1}}(s_h^k, a_h^k), \btheta_h^*\big\ra,
% \end{align*}
% where the first equation holds due to the definition of linear mixture MG, the second equation holds due to the definition of $\bphi_{\overline{V}_{k, h+1}}(\cdot, \cdot, \cdot)$. This indicates that the expectation of $\overline{V}_{k,h+1}(s_{h+1}^{k})$ is a linear function of $\bphi_{\overline{V}_{k,h+1}}(s_{h}^{k},a_{h}^{k},b_{h}^{k})$.
However, such an estimator is somehow limited since it treats each context-target pair equally and ignore the difference between these pairs. In principle, one should pay \emph{more} attention to the pairs with \emph{less} target variance since they carry more information about the unknown parameter $\btheta_h^*$. This observation inspires us to adapt the recently proposed  \emph{weighted ridge regression} scheme by \cite{zhou2020nearly} to estimate $\btheta_{h}^*$:
\begin{align}
    \overline{\btheta}_{k,h}^{(0)} &= \argmin_{\btheta \in \RR^d}\lambda\|\btheta\|_2^2 + \sum_{j = 1}^{k-1} \big[\big\la\bphi_{\overline{V}_{j, h+1}}(s_h^j, a_h^j,b_h^j), \btheta\big\ra- \overline{V}_{j, h+1}(s_{h+1}^j)\big]^2/\overline{\sigma}_{j,h}^2,\label{eq:theta0}
\end{align}
where $\overline{\sigma}_{j,h}^{2}$ is an appropriate upper bound on the variance of the value function $[\var_h\overline{\vvalue}_{j, h+1}](s_h^j, a_h^j, b_h^j)$. In particular, we construct $\overline{\sigma}_{k,h}^{2}$ as follows 
\begin{align}
\overline{\sigma}_{k,h} = \sqrt{\max\{H^{2}/d, \mathbb{V}^{\text{est}}\overline{V}_{k,h+1}(s_{h}^{k},a_{h}^{k},b_{h}^{k}) + \overline{E}_{k,h}\}}\label{eq:smallsigma}, 
\end{align}
where $[\var_{k,h}^{\text{est}}\overline{V}_{k, h+1}](s_h^k, a_h^k,b_h^k)$ is a scalar-valued empirical estimate for the variance of the value function $\overline{V}_{k, h+1}$
under the transition probability $\PP_h(\cdot|s_k^{h},a_k^{h},b_k^{h})$, and $\overline{E}_{k,h}$ is an offset term that is used to guarantee that
$\overline{\sigma}_{k,h}^{2}$ upper bounds $[\var_h\overline{\vvalue}_{k, h+1}](s_h^k, a_h^k, b_h^k)$ with high probability.

Weighted ridge regression \eqref{eq:theta0} has a closed-form solution $\overline{\btheta}_{k,h}^{(0)} = \big[\overline{\bSigma}_{k,h}^{(0)}\big]^{-1}\overline{\bbb}_{k,h}^{(0)}$, where the covariance matrix $\overline{\bSigma}_{k,h}^{(0)}$ can be computed by recursion starting at $\overline{\bSigma}_{1,h}^{(0)} = \lambda \Ib$:
\begin{align}
\overline{\bSigma}_{j+1,h}^{(0)} &=  \overline{\sigma}_{j,h}^{-2} \bphi_{\overline{V}_{j,h+1}}(s_{h}^{j},a_{h}^{j},b_{h}^{j})\bphi_{\overline{V}_{j,h+1}}(s_{h}^{j},a_{h}^{j},b_{h}^{j})^{\top} + \overline{\bSigma}_{j,h}^{(0)}\label{eq:BigSigma0},
\end{align}
and the correlation vector $\overline{\bbb}_{k,h}^{(0)}$ can be computed by recursion starting at $\overline{\bbb}_{1,h}^{(0)} = \zero$:
\begin{align}
\overline{\bbb}_{j+1,h}^{(0)} = \overline{\bbb}_{j,h}^{(0)} + \overline{\sigma}_{j,h}^{-2} \bphi_{\overline{V}_{j,h+1}}(s_{h}^{j},a_{h}^{j},b_{h}^{j})\overline{V}_{j,h+1}(s_{h+1}^{j})\label{eq:bvector0}.
\end{align}
By using a Bernstein-type self-normalized concentration inequality for vector-valued martingales proposed in \citet{zhou2020nearly}, one can show then that, with high probability, $\btheta^*_h$ lies in the ellipsoid $\overline{\cC}_{k,h}^{(0)}$ defined in \eqref{eq:Upepp}, where $\beta_k^{(0)}$ is the confidence radius chosen later in Lemma \ref{Thm:inset}.

\noindent\textbf{Variance Estimator.}
It remains to set $\overline{\sigma}_{j,h}^{2}$. We need to specify how to calculate the empirical variance  $[\var_{k,h}^{\text{est}}\overline{V}_{k, h+1}](s_h^k, a_h^k,b_h^k)$ and select $\overline{\error}_{k,h}$ to guarantee $\overline{\sigma}_{j,h}^{2}$ upper bounds $[\var_h\overline{\vvalue}_{k, h+1}](s_h^k, a_h^k, b_h^k)$ with high probability.
Recall the definition of $[\mathbb{V}_{h}V](\cdot,\cdot,\cdot)$ as follows: 
\begin{align*}
[\var_h\overline{\vvalue}_{k, h+1}](s_h^k, a_h^k, b_h^k)
&= [\PP_h\overline{\vvalue}_{k, h+1}^2](s_h^k, a_h^k, b_h^k) - \big([\PP_h\overline{\vvalue}_{k, h+1}](s_h^k, a_h^k, b_h^k)\big)^2\notag \\
& = \underbrace{\big\la\bphi_{\overline{\vvalue}_{k,h+1}^2}(s_h^k, a_h^k, b_h^k), \btheta_h^*\big\ra}_{I_{1}}  - \underbrace{\big[\big\la \bphi_{\overline{\vvalue}_{k,h+1}}(s_h^k, a_h^k, b_h^k), \btheta_h^*\big\ra\big]^2}_{I_{2}}.
\end{align*}
where the second equality holds due to the definition of linear mixture MGs. Notice that the expectation of $\overline{\vvalue}_{k, h+1}^2(s_{h+1}^k)$ over the next state, $s_{h+1}^k$, is a linear function of $\bphi_{\overline{\vvalue}_{k,h+1}^2}(s_h^k, a_h^k, b_h^k)$. Therefore, we use $\la \bphi_{\overline{\vvalue}_{k,h+1}^{2}}(s_h^k, a_h^k,b_h^k), \btheta_{k,h}^{(1)}\ra$ to estimate the term $I_{1}$ where $\btheta_{k,h}^{(1)}$ is the solution to the following ridge regression problem: 
\begin{align}
\btheta_{k,h}^{(1)} &= \argmin_{\btheta \in \RR^d}\lambda\|\btheta\|_2^2 + \sum_{j = 1}^{k-1} \big[\big\la\bphi_{\overline{\vvalue}_{j, h+1}^2}(s_h^j, a_h^j,b_h^j), \btheta\big\ra - \overline{\vvalue}_{j, h+1}^2(s_{h+1}^j)\big]^2\label{eq:theta1}.
\end{align}
For term $I_{2}$, we can use  $\langle\bphi_{\overline{V}_{k,h+1}}(s_{h}^{k},a_{h}^{k},b_{h}^{k}), \overline{\btheta}^{(0)}_{k,h}\rangle$ to estimate it. Thus we have the following variance estimator,
\begin{align}\label{eq:Vestimator}
&\mathbb{V}^{\text{est}}\overline{V}_{k,h+1}(s_{h}^{k},a_{h}^{k},b_{h}^{k})\leftarrow \big[ \langle\bphi_{\overline{V}^{2}_{k,h+1}}(s_{h}^{k},a_{h}^{k},b_{h}^{k}), \overline{\btheta}^{(1)}_{k,h}\rangle \big]_{[0,H^{2}]} - \big[\langle\bphi_{\overline{V}_{k,h+1}}(s_{h}^{k},a_{h}^{k},b_{h}^{k}), \overline{\btheta}^{(0)}_{k,h}\rangle\big]^{2}_{[-H,H]},
\end{align}
where the projection is used to control the range of our variance estimator. Lastly, we can compute $\btheta_{k,h}^{(1)}$ in a closed form $\overline{\btheta}_{k,h}^{(1)} = \big[\overline{\bSigma}_{k,h}^{(1)}\big]^{-1}\overline{\bbb}_{k,h}^{(1)}$, where the covariance matrix $\overline{\bSigma}_{k,h}^{(1)}$ is updated recursively in the following way:
\begin{align}
\overline{\bSigma}_{j+1,h}^{(1)} =  \overline{\bSigma}_{j,h}^{(1)} +  \bphi_{\overline{V}_{j,h+1}^{2}}(s_{h}^{j},a_{h}^{j},b_{h}^{j})\bphi_{\overline{V}_{j,h+1}^{2}}(s_{h}^{j},a_{h}^{j},b_{h}^{j})^{\top}\label{eq:BigSigma1},
\end{align}
and the correlation vector $\overline{\bbb}_{k,h}^{(1)}$ is updated in the following recursive form:
\begin{align}
\overline{\bbb}_{j+1,h}^{(1)} = \overline{\bbb}_{j,h}^{(1)} + \bphi_{\overline{V}_{j,h+1}^{2}}(s_{h}^{j},a_{h}^{j},b_{h}^{j})\overline{V}_{j,h+1}^{2}(s_{h+1}^{j}).\label{eq:bvector1}
\end{align}
 By the standard self-normalized concentration inequality for vector-valued martingales in \citet{abbasi2011improved}, we can show that, with high probability, $\overline{\sigma}_{j,h}^{2}$ upper bounds $[\var_h\overline{\vvalue}_{k, h+1}](s_h^k, a_h^k, b_h^k)$ if we select $\overline{E}_{k,h}$ as follows
\begin{align}\label{eq:Eesptimator}
\overline{E}_{k, h} 
&= \min\big\{H^{2}, \beta^{(1)}_{k}\|\big[\overline{\bSigma}_{k, h}^{(1)}\big]^{-1/2}\bphi_{\overline{V}_{k, h+1}^2}(s_{h}^{k}, a_{h}^{k},b_{h}^{k})\|_{2}\big\}\notag\\
& \qquad+ \min\big\{H^{2}, 2H\beta^{(2)}_{k}\|\big[\overline{\bSigma}_{k, h}^{(0)}\big]^{-1/2} \bphi_{\overline{V}_{k, h+1}}(s_{h}^{k}, a_{h}^{k},b_{h}^{k})\|_{2}\big\},
\end{align}
where $\beta_{k}^{(1)},\beta_{k}^{(2)}$ are constants chosen later in Lemma \ref{Thm:inset}. 

\begin{remark}
Our $\algname$ is computational efficient for specific feature mapping $\bphi$, as \citet{ayoub2020model, zhou2020nearly} suggested. For a special class of $\bphi$, where
\begin{align}
    \bphi(s'|s,a,b) = \bpsi(s')\odot \bmu(s,a,b),\ \bpsi(\cdot): \cS \rightarrow \RR^d,\ \bmu(\cdot, \cdot, \cdot): \cS \times \cA_{\text{max}} \times \cA_{\text{min}}\rightarrow \RR^d,\notag
\end{align}
$\odot$ is the componentwise product, $\algname$ can be implemented within $\text{poly}(d, |\cA_{\text{max}}|, |\cA_{\text{min}}|)\cdot KH$ time complexity with the access to some integration oracle $\cO$. The details are deferred to Appendix \ref{sec:compute}.
\end{remark}

\noindent\textbf{Difference between $\algname$ and previous algorithms}
 Here we compare our $\algname$ with the OMNI-V proposed by \citet{xie2020learning}. First, \citet{xie2020learning} studied the linear MGs while we study the linear mixture MGs. Second, due to the difference between the studied models, OMNI-V needs to maintain a covering set of the estimated Q functions (Eq. (5), \citealt{xie2020learning}), which makes its space complexity exponential in $d$. In sharp contrast, our $\algname$ relies on the value targeted regression \citep{jia2020model, ayoub2020model} and does not need to maintain such a cover set.

% \subsection{Computational efficiency}
% Recall the definition for CCE in \citet{xie2020learning, bai2020near}, for any given $s \in \cS$ and pair of matrix $(P,Q)\in |\cA_{\max}| \times |\cA_{\min}|$, the subroutine CCE $(P, Q)$ returns policy $(\pi, \nu)$ such that 
% \begin{align*}
% \EE_{a\sim\pi,b\sim\nu}P(s,a,b) &\geq \max_{a}\EE_{b\sim\nu}P(s,a,b)\\
% \EE_{a\sim\pi,b\sim\nu}Q(s,a,b)&\leq \min_{b}\EE_{a\sim \pi}Q(s,a,b).
% \end{align*}

% A CCE always exists because a Nash equilibrium for a general-sum game with payoff matrices $(P, Q)$ is also a CCE defined by $(P, Q)$, and a Nash equilibrium always exists. Notice A CCE is also a CCE with $\epsilon$ tolerant, so we have prove the existence. Moreover we assume the action space is finite, the constraints for CCE can be rewritten as $|\cA_{\max}| \times |\cA_{\min}|$ linear constraints, which can be efficiently resolved by some standard linear programming algorithm \citep{carmon2020coordinate}.

\section{Main Results}\label{sec:Main theorem for offline}
In this section, we present the main theoretical results. We first show that under a specific parameter choice, our constructed confidence sets $\overline{\cC}_{k,h}^{(0)}$ and $\underline{\cC}_{k,h}^{(0)}$ include $\btheta^*_h$ with high probability, and the estimated variances $\mathbb{V}^{\text{est}}\overline{V}_{k,h+1}(s_{h}^{k},a_{h}^{k},b_{h}^{k})$ and $\mathbb{V}^{\text{est}}\underline{V}_{k,h+1}(s_{h}^{k},a_{h}^{k},b_{h}^{k})$ deviate from the true variances by at most the offset terms $\overline{E}_{k,h},\underline{E}_{k,h}$.  
\begin{lemma}\label{Thm:inset}
Setting $\beta_{k}^{(0)}$ in \eqref{eq:Upepp} and $\beta_{k}^{(1)},\beta_{k}^{(2)}$ in \eqref{eq:Eesptimator} to
\begin{align*} 
\beta_{k}^{(0)} &= 16\sqrt{d\log(1+k/\lambda)\log(4k^{2}H/\delta)}+ 8\sqrt{d}\log(4k^{2}H/\delta) + \sqrt{\lambda }B\\  
\beta_{k}^{(1)} &= 16\sqrt{dH^{4}\log(1+KH^{4}/(d\lambda))\log(4k^{2}H/d\delta)} + 8H^{2}\log(4k^{2}H/\delta) + \sqrt{\lambda}B\\
\beta_{k}^{(2)} &=  16d\sqrt{\log(1+k/\lambda)\log(4k^{2}H/\delta)} + 8\sqrt{d}\log(4k^{2}H/\delta)+\sqrt{\lambda}B,
\end{align*}then with probability at least $1-3\delta$, we have $\btheta_{h}^{*}\in \overline{\cC}_{k,h}^{(0)}\cap  \underline{\cC}_{k,h}^{(0)}$.
In addition, we have
\begin{align*}
|\mathbb{V}^{\text{est}}\overline{V}_{k,h+1}(s_{h}^{k},a_{h}^{k},b_{h}^{k}) -   \mathbb{V}\overline{V}_{k,h+1}(s_{h}^{k},a_{h}^{k},b_{h}^{k})| &\leq \overline{E}_{k,h}\\
|\mathbb{V}^{\text{est}}\underline{V}_{k,h+1}(s_{h}^{k},a_{h}^{k},b_{h}^{k}) -   \mathbb{V}\underline{V}_{k,h+1}(s_{h}^{k},a_{h}^{k},b_{h}^{k})| &\leq \underline{E}_{k,h}
\end{align*}
\end{lemma}

%\subsection{Upper Bound}
Next, we present the regret of $\algname$. 
\begin{theorem}\label{mainthm}
Setting $\lambda = 1/B^{2}$, $\epsilon = O(HT^{-1/2})$, then with probability at least $1-5\delta$,  the regret of Algorithm \ref{alg:Offline} $\text{Regret}(M_{\btheta^{*}}, K)$ is bounded by
\begin{align*}
\tilde{\cO}\big(\sqrt{d^{2}H^{2}+dH^{3}}\sqrt{T}+d^{2}H^{3}+d^{3}H^{2}\big),
\end{align*}
where $T=KH$.
\end{theorem}

%\begin{remark}
Theorem \ref{mainthm} suggest that when $d\geq H$ and $T \geq d^{4}H^{2}$, the regret of $\algname$ is bounded by $\tilde O(dH\sqrt{T})$. 
%\end{remark}

\begin{remark}
Our $\algname$ also enjoys a finite sample complexity. By the standard online-to-batch conversion, we can show that $\algname$ is guaranteed to find an $\epsilon$-approximate NE, i.e., $(\pi, \nu)$ satisfying $V_1^{*, \nu} - V_1^{\pi, *} \leq \epsilon$, within $\tilde O((d^2H^3 + dH^4)/\epsilon^2)$ episodes.
\end{remark}

\begin{remark}\label{Remark1}
% \todod{add comparison with the tabular case}
We can apply our algorithm to tabular MGs and our results can be reduced to the setting with $|S| = S, |\cA_{\max}| = A, |\cA_{\min}| = B$ by choosing $\bphi(s'|s,a,b)$ as the one-hot representation of $\PP(s'|s,a,b)$. It is easy to verify that \eqref{mapping} holds and $d = S^2AB$. Thus the regret bound given in Theorem \ref{mainthm} reduces to $\tilde{O}(\sqrt{S^{4}H^{2}A^{2}B^{2}T})$, which does not match the lower bound of tabular MGs in \citet{bai2020provable}. %However, when applied to the tabular setting, our algorithm is similar to the algorithm Nash-VI \citep{liu2020sharp} which also benefit from Bernstein-type bonus. 
We would like to point out that by using some techniques specialized to the tabular setting, the regret bound of our algorithm for tabular MGs can be improved, which is beyond the scope of this work. %Besides, recovering the state-of-art regret bound for the tabular setting is not the focus of this paper, and our algorithm and analysis can tackle a much more general setting.
\end{remark}

%\subsection{Lower Bound}
Here, we present a lower bound for linear mixture MGs. It has been shown in \citet{zhou2020nearly} that the regret lower bound for learning linear mixture MDPs is $\Omega(dH\sqrt{T})$, from which we can prove a lower bound for learning linear mixture MGs, since MDPs can be regarded as a special case of MGs with one dummy player, i.e., $\PP_h(s'|s,a,b) = \PP_h(s'|s,a)$ and $r_h(s,a,b) = r_h(s,a)$. Formally, we have the following lower bound:
\begin{theorem}[Regret lower bound]\label{theorem:lowerbound}
Let $B>1$ and $K \geq \max\{(d-1)^2H/2, (d-1)/(32H(B - 1))\}$, $d \geq 4$, $H\geq 3$. Then for any algorithm there exists an episodic, $B$-bounded linear mixture MG $M_{\btheta^{*}}$ such that the expected regret of first $T$ rounds is lower bounded as follows: 
% \todoc{I guess the rewards should be normalized..}
\begin{align}
\EE[\text{Regret}(M_{\btheta^{*}}, K)] \geq \Omega\big(dH\sqrt{T}\big),\notag
\end{align}
where $T= KH$.
\end{theorem}
\begin{remark}
When $d\geq H$ and $T \geq d^{4}H^{2}$, the regret of $\algname$ matches the lower bound up to logarithmic factors. Therefore, $\algname$ is nearly minimax optimal. 
\end{remark}
\begin{remark}
Based on a similar argument made in \citet{zhou2020nearly}, we can show that the same lower bound holds for the Markov games with linear structures. Recall that the best-known algorithm for learning MGs with linear structures is OMNI-VI \citep{xie2020learning}, which has an $\tilde O(\sqrt{d^3H^3T})$ regret. This suggests that there is still a gap that needs to be closed for learning MGs with linear structure. Please see the appendix for more details. %whi applies to the linear MG setting considered by \citet{xie2020learning}. 
\end{remark}

Turn-based linear mixture MG can be regarded as a special case of linear mixture simultaneous-move MG. Therefore, we can still use Algorithm \ref{alg:Offline} to find the Nash equilibrium and then by Theorem \ref{mainthm}, we can further show that the regret of our turn-based algorithm is also bounded by $\tilde O(dH\sqrt{T})$. Notice that for the turn-based game, at each step only one player can take action. Thus, the $\epsilon$-CCE routine in Line \ref{alg:CCE} of Algorithm \ref{alg:Offline} needs be replaced by two separate subroutines: taking $\pi_h^k$ and $\nu_h^k$ as greedy policies w.r.t. $\overline{Q}_{k,h}$ and  $\underline{Q}_{k,h}$. For completeness, we present the turn-based version of Algorithm \ref{alg:Offline} as Algorithm \ref{alg:turnbased} 
in Appendix \ref{sec:turn}.

\section{Conclusions and Future Work}
In this paper, we proposed the first provably optimal algorithm for learning two-player zero-sum Markov games with linear function approximation and without assuming access to the generative model. Specifically, we show that $\algname$ can provably achieve an $\tilde{O}(dH\sqrt{T})$ regret, where $d$ is the linear function dimension, $H$ is the length of the game/episode, and $T$ is the total number of steps in the Markov game. We also prove an $\tilde{\Omega}( dH\sqrt{T})$ lower bound on the regret. Our upper bound matches the lower bound up to logarithmic factors, which suggests the optimality of our algorithm.

There are several important future directions. First, in the current linear mixture MG, the feature mapping encodes the information of both players. To reproduce the difference between $(A+B)$ and $AB$ in the tabular setting, we may need to construct a separate feature mapping for each player \citep{bai2020near}. Second, while our algorithms can be extended to the decentralized setting, it is not clear if the minimax-optimal regret can still be obtained because of the adversarial policy. How to achieve a near-optimal decentralized algorithm is another important future work.

% Acknowledgments---Will not appear in anonymized version
\acks{We thank the anonymous reviewers for their helpful comments. 
Part of this work was done when ZC, DZ and QG participated the Theory of Reinforcement Learning program at the Simons Institute for the Theory of Computing in Fall 2020. ZC, DZ and QG are partially supported by the National Science Foundation IIS-2008981, CAREER Award 1906169, and AWS Machine Learning Research Award. The views and conclusions contained in this paper are those of the authors and should not be interpreted as representing any funding agencies.}

\bibliography{twoplayer}

\appendix

\section{Computational Efficiency of $\algname$}\label{sec:compute}
As \citet{ayoub2020model, zhou2020nearly} suggested, the computational efficiency of $\algname$ will depend on the feature mapping $\bphi(s'|s,a,b)$. In this section we show that for a specific family of $\bphi$ with the access to some integration oracle $\cO$, $\algname$ can be implemented within polynomial computational complexity. We consider a special class of $\bphi$, where
\begin{align}
    \bphi(s'|s,a,b) = \bpsi(s')\odot \bmu(s,a,b),\ \bpsi(\cdot): \cS \rightarrow \RR^d,\ \bmu(\cdot, \cdot, \cdot): \cS \times \cA_{\text{max}} \times \cA_{\text{min}}\rightarrow \RR^d,\notag
\end{align}
$\odot$ is the componentwise product. Meanwhile, we assume that there exists an oracle $\cO$ such that for any function $V:\cS \rightarrow \RR$, the summation $\sum_{s}\bpsi(s)V(s)$ can be evaluated by considering at most $p(d)$ number of states $s \in \cS$. From now on we show how to compute each key step in $\algname$. First, to compute $\overline{Q}_{k,h}$, note that $\overline{Q}_{k,h}$ can be parameterized by $\hat\btheta_{k,h}^{(0)}$ and $\hat{\bSigma}_{k,h}^{(0)}$ as follows:
\begin{align}
    &\overline{Q}_{k,h}(\cdot, \cdot, \cdot) \notag \\
    &= \Big[ r_{h}(\cdot,\cdot,\cdot) + \langle \overline{\btheta}^{(0)}_{k,h}, \bphi_{\overline{V}_{k,h+1}}(\cdot,\cdot,\cdot) \rangle + \beta^{(0)}_{k}\Big\|\big[\overline{\bSigma}_{k,h}^{(0)}\big]^{-1/2}\bphi_{\overline{V}_{k,h+1}}(\cdot,\cdot,\cdot)\Big\|_{2}\Big]_{[-H,H]}\notag \\
    & = \Big[ r_{h}(\cdot,\cdot,\cdot) + \la\underbrace{\overline{\btheta}^{(0)}_{k,h}\odot \Big(\sum_{s'}\bpsi(s') \overline{V}_{k,h+1}(s')\Big)}_{\hat\btheta_{k,h}^{(0)}}, \bmu(\cdot, \cdot, \cdot)\ra+ \beta^{(0)}_{k} \Big\| \hat{\bSigma}_{k,h}^{(0)}\bmu(\cdot, \cdot, \cdot)\Big\|_{2}\Big]_{[-H,H]},\label{help:1}
\end{align}
where the $(i,j)$-th entry of $\hat{\bSigma}_{k,h}^{(0)}$ is $\big[\overline{\bSigma}_{k,h}^{(0)}\big]^{-1/2}_{i,j}[\sum_{s'}\psi_j(s')\overline{V}_{k,h+1}(s')]$. Given $\hat\btheta_{k,h+1}^{(0)}$ and $\hat{\bSigma}_{k,h+1}^{(0)}$, we need $O(d^2)$ to compute $\overline{Q}_{k,h}$, which is the same as computing $\underline{Q}_{k,h}$. Then, for each $s$, we need $O(|\cA_{\text{max}}||\cA_{\text{min}}|)$ complexity to compute $\mu_h^k(\cdot, \cdot|s)$ and $O(d^2|\cA_{\text{max}}||\cA_{\text{min}}|)$ complexity to compute $\overline{V}_{k,h}$ and $\underline{V}_{k,h}$. Finally, to obtain $\hat\btheta_{k,h}^{(0)}$ and $\hat{\bSigma}_{k,h}^{(0)}$, we need to evaluate $\overline{V}_{k,h}$ over $p(d)$ states, which by \eqref{help:1}, requires $O(p(d)d^2|\cA_{\text{max}}||\cA_{\text{min}}|)$ complexity in total. Therefore, we need $\text{poly}(d, |\cA_{\text{max}}|, |\cA_{\text{min}}|)$ complexity to compute one $\overline{Q}_{k,h}$, and we need $\text{poly}(d, |\cA_{\text{max}}|, |\cA_{\text{min}}|)\cdot KH$ complexity for implementing $\algname$, given all $\overline{\btheta}_{k,h}^{(0)}$, $\underline{\btheta}_{k,h}^{(0)}$. The complexity of computing $\overline{\btheta}_{k,h}^{(0)}$ includes the complexity to solve the regression problem \eqref{eq:theta0}, \eqref{eq:theta1} and to compute the variance estimator $\overline{\sigma}_{k,h}$, $\underline{\sigma}_{k,h}$, which again is at most $\text{poly}(d, |\cA_{\text{max}}|, |\cA_{\text{min}}|)\cdot KH$ according to previous analysis. Therefore, the total complexity of implementing $\algname$ is $\text{poly}(d, |\cA_{\text{max}}|, |\cA_{\text{min}}|)\cdot KH$.

\section{Proof of Results in Section \ref{sec:Main theorem for offline}}\label{sec:AppendixA}

We let $\PP$ be the distribution over $(\cS \times \cA_{\max} \times \cA_{\min})^{\NN}$ induced by the episodic MG $M$, and further denote the sample space $\Omega=(\cS\times \cA_{\max} \times \cA_{\min})^{\NN}$. Thus, we work with the probability space given by the triplet $(\Omega,\cF,\PP)$, where $\cF$ is the product $\sigma$-algebra generated by the discrete $\sigma$-algebras underlying $\cS$,  $\cA_{\max}$ and $\cA_{\min}$.

For $1\le k \le K$, $1\le h \le H$, let $\cF_{k,h}$ be the $\sigma$-algebra generated by the random variables representing the state-action-action pairs up to and including those that appear stage $h$ of episode $k$. That is, $\cF_{k,h}$ is generated by
\begin{align*}
s_1^1,a_1^1, b_1^1, \dots, s_h^1,a_h^1,b_h^1, &\dots, s_H^1,a_H^1, b_H^1\,, \\
s_1^2,a_1^2, b_1^2,  \dots, s_h^2,a_h^2, b_h^2,&\dots, s_H^2,a_H^2, b_H^2\,, \\
\vdots \\
s_1^k,a_1^k,b_1^k,\dots, s_h^k,a_h^k,b_h^k & \,.
\end{align*}

\subsection{Proof of Lemma \ref{Thm:inset}}
For simplicity we denote the following confident sets:
\begin{align}
&\overline{\cC}_{k,h}^{(0)} = \bigg\{\btheta: \Big\|\big[\overline{\bSigma}_{k,h}^{(0)}\big]^{1/2}(\btheta - \overline{\btheta}_{k,h}^{(0)})\Big\|_2 \leq \beta_k^{(0)}\bigg\},\underline{\cC}_{k,h}^{(0)} = \bigg\{\btheta: \Big\|\big[\underline{\bSigma}_{k,h}^{(0)}\big]^{1/2}(\btheta - \underline{\btheta}_{k,h}^{(0)})\Big\|_2 \leq \beta_k^{(0)}\bigg\},\notag \\
& \overline{\cC}_{k,h}^{(1)} = \bigg\{\btheta: \Big\|\big[\overline{\bSigma}_{k,h}^{(1)}\big]^{1/2}(\btheta - \overline{\btheta}_{k,h}^{(1)})\Big\|_2 \leq \beta_k^{(1)}\bigg\},\underline{\cC}_{k,h}^{(1)} = \bigg\{\btheta: \Big\|\big[\underline{\bSigma}_{k,h}^{(1)}\big]^{1/2}(\btheta - \underline{\btheta}_{k,h}^{(1)})\Big\|_2 \leq \beta_k^{(1)}\bigg\},\notag\\
&\overline{\cC}_{k,h}^{(2)} = \bigg\{\btheta: \Big\|\big[\overline{\bSigma}_{k,h}^{(0)}\big]^{1/2}(\btheta - \overline{\btheta}_{k,h}^{(0)})\Big\|_2 \leq \beta_k^{(2)}\bigg\},\underline{\cC}_{k,h}^{(2)} = \bigg\{\btheta: \Big\|\big[\underline{\bSigma}_{k,h}^{(0)}\big]^{1/2}(\btheta - \underline{\btheta}_{k,h}^{(0)})\Big\|_2 \leq \beta_k^{(2)}\bigg\}.\notag 
\end{align}
By the selection $\beta_{k}^{(0)} < \beta_{k}^{(2)}$ in Lemma \ref{Thm:inset}, we have that $\overline{\cC}_{k,h}^{(0)} \subset  \overline{\cC}_{k,h}^{(2)}$ and $\underline{\cC}_{k,h}^{(0)} \subset  \underline{\cC}_{k,h}^{(2)}$.
We first use  standard self-normalized tail inequality to show that $\btheta_h^*$ is included in $\overline{\cC}_{k,h}^{(1)}\cap  \overline{\cC}_{k,h}^{(2)}$ with high probability. Based on that
we can further decrease 
$\beta_k^{(2)}$ to $\beta_{k}^{(1)}$
without significantly increasing the probability of the bad event 
when $\btheta_h^* \not\in \overline{\cC}_{k,h}^{(0)}$ or $\btheta_h^* \not\in \underline{\cC}_{k,h}^{(0)}$.

We start with the following Bernstein-type self-normalized concentration inequality.

\begin{lemma}[Theorem 2, \citealt{zhou2020nearly}]\label{lemma:concentration_variance}
Let $\{\cG_{t}\}_{t=1}^\infty$ be a filtration, $\{\xb_t,\eta_t\}_{t\ge 1}$ a stochastic process so that
$\xb_t \in \RR^d$ is $\cG_t$-measurable and $\eta_t \in \RR$ is $\cG_{t+1}$-measurable. 
Fix $R,L,\sigma,\lambda>0$, $\bmu^*\in \RR^d$. 
For $t\ge 1$ 
let $y_t = \la \bmu^*, \xb_t\ra + \eta_t$ and
suppose that $\eta_t, \xb_t$ also satisfy 
\begin{align}
    |\eta_t| \leq R,\ \EE[\eta_t|\cG_t] = 0,\ \EE [\eta_t^2|\cG_t] \leq \sigma^2,\ \|\xb_t\|_2 \leq L.\notag
\end{align}
Then, for any $0 <\delta<1$, with probability at least $1-\delta$ we have 
\begin{align}
    \forall t>0,\ \bigg\|\sum_{i=1}^t \xb_i \eta_i\bigg\|_{\Zb_t^{-1}} \leq \beta_t,\ \|\bmu_t - \bmu^*\|_{\Zb_t} \leq \beta_t + \sqrt{\lambda}\|\bmu^*\|_2,\label{eq:concentration_variance:xx}
\end{align}
where for $t\ge 1$,
 $\bmu_t = \Zb_t^{-1}\bbb_t$,
 $\Zb_t = \lambda\Ib + \sum_{i=1}^t \xb_i\xb_i^\top$,
$\bbb_t = \sum_{i=1}^ty_i\xb_i$
 and
 \[
\beta_t = 8\sigma\sqrt{d\log(1+tL^2/(d\lambda)) \log(4t^2/\delta)} + 4R \log(4t^2/\delta)\,.
\]
\end{lemma}

\begin{lemma}\label{lemma: VB}
For every $1 \leq k \leq K$ and $1 \leq h \leq H$, we have
\begin{align*}
&|\mathbb{V}^{\text{est}}\overline{V}_{k,h+1}(s_{h}^{k},a_{h}^{k},b_{h}^{k}) -   \mathbb{V}\overline{V}_{k,h+1}(s_{h}^{k},a_{h}^{k},b_{h}^{k})|\\
&\leq \min\Big\{H^{2}, \Big\|\big[\overline{\bSigma}_{k,h}^{(1)}\big]^{-1/2}\bphi_{\overline{V}^{2}_{k,h+1}}(s_{h}^{k},a_{h}^{k},b_{h}^{k})\Big\|_{2}\Big\|\big[\overline{\bSigma}^{(1)}_{k,h}\big]^{1/2}(\overline{\btheta}^{(1)}_{k,h} - \btheta^{*}_{h})\Big\|_{2}\Big\}\\
&\qquad + \min\Big\{H^{2}, 2H\Big\|\big[\overline{\bSigma}^{(0)}_{k,h}\big]^{-1/2}\bphi_{\overline{V}^{2}_{k,h+1}}(s_{h}^{k},a_{h}^{k},b_{h}^{k})\Big\|_{2}\Big\|\big[\overline{\bSigma}_{k,h}^{(0)}\big]^{1/2}(\overline{\btheta}^{(0)}_{k,h} - \btheta^{*}_{h})\Big\|_{2}\Big\},
\end{align*}
and
\begin{align*}
&|\mathbb{V}^{\text{est}}\underline{V}_{k,h+1}(s_{h}^{k},a_{h}^{k},b_{h}^{k}) -   \mathbb{V}\underline{V}_{k,h+1}(s_{h}^{k},a_{h}^{k},b_{h}^{k})|\\
&\leq \min\Big\{H^{2}, \Big\|\big[\underline{\bSigma}_{k,h}^{(1)}\big]^{-1/2}\bphi_{\underline{V}^{2}_{k,h+1}}(s_{h}^{k},a_{h}^{k},b_{h}^{k})\Big\|_{2}\Big\|\big[\underline{\bSigma}^{(1)}_{k,h}\big]^{1/2}(\underline{\btheta}^{(1)}_{k,h} - \btheta^{*}_{h})\Big\|_{2}\Big\}\\
&\qquad + \min\Big\{H^{2}, 2H\Big\|\big[\underline{\bSigma}^{(0)}_{k,h}\big]^{-1/2}\bphi_{\underline{V}^{2}_{k,h+1}}(s_{h}^{k},a_{h}^{k},b_{h}^{k})\Big\|_{2}\Big\|\big[\underline{\bSigma}_{k,h}^{(0)}\big]^{1/2}(\underline{\btheta}^{(0)}_{k,h} - \btheta^{*}_{h})\Big\|_{2}\Big\}.
\end{align*}
\end{lemma}
\begin{proof}[Proof of Lemma \ref{Thm:inset}]
For simplicity, we only prove the results for the max-player. Fix $h\in[H]$. 

We first show that with probability at least $1-\delta/(2H)$, $\Big\|\big[\overline{\bSigma}^{(0)}_{k,h}\big]^{1/2}(\overline{\btheta}^{(0)}_{k,h} - \btheta^{*}_{h})\Big\|_{2} \leq \beta^{(2)}_{k}$.To show this, we apply Lemma \ref{lemma:concentration_variance}. Let $\xb_i = \overline{\sigma}_{i,h}^{-1}\bphi_{\overline{V}_{i, h+1}}(s_h^i, a_h^i, b_h^i)$ and $\eta_i =\overline{\sigma}_{i,h}^{-1}\overline{V}_{i, h+1}(s_{h+1}^i) - \overline{\sigma}_{i,h}^{-1}\la\bphi_{\overline{V}_{i, h+1}}(s_{h}^{i}, a_{h}^{i}, b_{h}^{i}), \btheta^*_h \ra$, $\cG_i = \cF_{i,h}$, $\bmu^* = \btheta^*_h$, $y_i = \la \bmu^*, \xb_i\ra + \eta_i$, $\Zb_i = \lambda \Ib + \sum_{i' = 1}^i \xb_{i'}\xb_{i'}^\top$, $\bbb_i  = \sum_{i' = 1}^i \xb_{i'}y_{i'}$ and $\bmu_i = \Zb_i^{-1}\bbb_i$. Then it can be verified that $y_i = \overline{\sigma}_{i,h}^{-1}\overline{V}_{i,h+1}(s_{h+1}^i)$ and $\bmu_{i} = \overline{\btheta}^{(0)}_{i+1, h}$. Moreover,  we have that
\begin{align}
    \|\xb_i\|_2 \leq \overline{\sigma}_{i,h}^{-1}H \leq \sqrt{d}
    ,\ \ 
    |\eta_i| \leq \overline{\sigma}_{i,h}^{-1}2H \leq 2\sqrt{d}
    ,\ \ 
    \EE[\eta_i|\cG_i] = 0
    ,\ \
    \EE[\eta_i^2|\cG_i] \leq 4d
    \,,\notag
\end{align}
where we apply $\| \bphi_{\overline{V}_{i,h+1}}(\cdot,\cdot,\cdot)\|_2 \le H$, $\overline{V}_{i,h+1} \in [-H,H]$ and $\overline{\sigma}_{i,h} \geq H/\sqrt{d}$.
Since we also have that $\bx_i$ is $\cG_i$ measurable and $\eta_i$ is $\cG_{i+1}$ measurable,
by Lemma \ref{lemma:concentration_variance}, we obtain that with probability at least $1-\delta/(2H)$, for all $k \leq K$, $\Big\|\big[\overline{\bSigma}^{(0)}_{k,h}\big]^{1/2}(\overline{\btheta}^{(0)}_{k,h} - \btheta^{*}_{h})\Big\|_{2}$ is bounded by 
\begin{align}
 16d\sqrt{\log(1+k/\lambda) \log(8k^2H/\delta)}+ 8\sqrt{d} \log(8k^2H/\delta) + \sqrt{\lambda}\pnorm = \beta_k^{(2)},\label{eq:concentrate:finite:1}
\end{align}
implying that with probability at least $1-\delta/(2H)$, for any $k \leq K$, $\btheta^*_h \in \overline{\cC}_{k,h}^{(2)}$. 

An argument, which is analogous to the one just used (except that now the range of the ``noise'' matches the range of ``squared values'' and is thus bounded by $H^2$, rather than being bounded by $\sqrt{d}$) gives that with probability at least $1-\delta/(2H)$, for any $k \leq K$ we have $ \Big\|\big[\overline{\bSigma}^{(1)}_{k,h}\big]^{1/2}(\overline{\btheta}^{(1)}_{k,h} - \btheta^{*}_{h})\Big\|_{2}$ bounded by
\begin{align}
 16\sqrt{dH^4\log(1+kH^4/(d\lambda)) \log(8k^2H/\delta)}+ 8H^2 \log(8k^2H/\delta) + \sqrt{\lambda}\pnorm = \beta_k^{(1)},\label{eq:concentrate:finite:1.11}
\end{align}
implying that with probability at least $1-\delta/(2H)$, for any $k \leq K$, $\btheta^*_h \in \overline{\cC}_{k,h}^{(1)}$. 

We now show that $\btheta^*_h \in \overline{\cC}_{k,h}^{(0)}$ with high probability. We again apply Lemma \ref{lemma:concentration_variance}. Let $\xb_i = \overline{\sigma}_{i,h}^{-1}\bphi_{\overline{V}_{i, h+1}}(s_h^i, a_h^i, b_h^i)$ and 
\begin{align}
    \eta_i = \overline{\sigma}_{i,h}^{-1}\ind\{\btheta^*_h \in \overline{\cC}_{i,h}^{(1)} \cap \overline{\cC}_{i,h}^{(2)} \}\big[\overline{V}_{i, h+1}(s_{h+1}^i) - \la\bphi_{\overline{V}_{i, h+1}}(s_h^i, a_h^i, b_h^i), \btheta^*_h \ra\big],\notag
\end{align}
$\cG_i = \cF_{i,h}$, $\bmu^* = \btheta^*_h$, $y_i = \la \bmu^*, \xb_i\ra + \eta_i$, $\Zb_i = \lambda \Ib + \sum_{i' = 1}^i \xb_{i'}\xb_{i'}^\top$, $\bbb_i  = \sum_{i' = 1}^i \xb_{i'}y_{i'}$ and $\bmu_i = \Zb_i^{-1}\bbb_i$. Still we have that $\|\xb_i\|_2 \leq \overline{\sigma}_{i,h}^{-1}H \leq \sqrt{d}$.Because $\ind\{\btheta^*_h \in \overline{\cC}_{i,h}^{(1)} \cap \overline{\cC}_{i,h}^{(2)} \}$ is $\cG_i$-measurable, we have $\EE[\eta_i|\cG_i] = 0$. We also have $|\eta_i| \leq \overline{\sigma}_{i,h}^{-1}2H \leq 2\sqrt{d}$ since $|\overline{V}_{i, h+1}(\cdot)| \leq H$ and $\overline{\sigma}_{i,h} \geq H/\sqrt{d}$. To get better bound $\beta_{k}^{(0)}$ rather than $\beta_{k}^{(2)}$ in \eqref{eq:concentrate:finite:1}, we need more careful computation of  $\EE[\eta_i^2|\cG_i]$ as follows,
\begin{align}
    &\EE[\eta_i^2|\cG_i] = \overline{\sigma}_{i,h}^{-2}\ind\{\btheta^*_h \in \overline{\cC}_{i,h}^{(1)} \cap \overline{\cC}_{i,h}^{(2)} \}[\var_h\overline{V}_{i, h+1}](s_h^i, a_h^i, b_h^i)\notag \\
    & \leq \overline{\sigma}_{i,h}^{-2}\ind\{\btheta^*_h \in \overline{\cC}^{(1)}_{i,h} \cap \overline{\cC}^{(2)}_{i,h} \}\bigg[[\var^{\text{est}}_{i,h}\overline{V}_{i,h+1}](s_h^i, a_h^i, b_h^i) \notag \\
    &\qquad + \min\Big\{H^2, \Big\|\big[\overline{\bSigma}_{i,h}^{(1)}\big]^{-1/2}\bphi_{\overline{V}_{i,h+1}^2}(s_h^i, a_h^i, b_h^i)\Big\|_2 \Big\|\big[\overline{\bSigma}_{i,h}^{(1)}\big]^{1/2}\big(\overline{\btheta}^{(1)}_{i,h} - \btheta^*_h\big)\Big\|_2\Big\}\notag \\
    &\qquad + \min\Big\{H^2,2H\Big\|\big[\overline{\bSigma}_{i,h}^{(0)}\big]^{-1/2}\bphi_{\overline{V}_{i,h+1}}(s_h^i, a_h^i, b_h^i)\Big\|_2 \Big\|\big[\overline{\bSigma}_{i,h}^{(0)}\big]^{1/2}\big(\overline{\btheta}^{(0)}_{i,h} - \btheta^*_h\big)\Big\|_2\Big\}\bigg]\notag \\
    & \leq \overline{\sigma}_{i,h}^{-2}\bigg[[\var_{i,h}^{\text{est}}\overline{V}_{i,h+1}](s_h^i, a_h^i, b_h^i)  + \min\Big\{H^2, \beta_i^{(1)}\Big\|\big[\bSigma_{i,h}^{(1)}\big]^{-1/2}\bphi_{\overline{V}_{i,h+1}^2}(s_h^i, a_h^i, b_h^i)\Big\|_2\Big\}\notag \\
    &\qquad + \min\Big\{H^2,2H\beta_i^{(2)}\Big\|\big[\bSigma_{i,h}^{(0)}\big]^{-1/2}\bphi_{\overline{V}_{i,h+1}}(s_h^i, a_h^i, b_h^i)\Big\|_2\Big\}\bigg]\notag \\
    & = 1, \notag
\end{align}
where the first inequality holds due to Lemma \ref{lemma: VB}, the second inequality holds due to the indicator function, the last equality holds due to the definition of $\overline{\sigma}_{i,h}$. 
Then, by Lemma \ref{lemma:concentration_variance}, with probability at least $1-\delta/(2H)$, $\forall k \leq K$, 
\begin{align}
    \|\bmu_k - \bmu^*\|_{\Zb_i} &\leq 16\sqrt{d\log(1+k/\lambda) \log(8k^2H/\delta)}+ 8\sqrt{d} \log(8k^2H/\delta) + \sqrt{\lambda}\pnorm = \beta_k^{(0)},\label{eq:concentrate:finite:2}
\end{align}
where the  equality uses the definition of $\beta_k^{(0)}$.
Let $\event'$ be the event when
 $\btheta^*_h \in \cap_{k\le K}\overline{\cC}^{(1)}_{k,h} \cap\overline{\cC}^{(2)}_{k,h}$ and
 \eqref{eq:concentrate:finite:2} hold.
By the union bound, $\PP(\event')\ge 1-3\delta/(2H)$.

We now show that $\btheta^*_h \in \overline{\cC}_{k,h}^{(0)}$ holds on $\event'$.
For this note that on $\event'$, 
for any $k \leq K$, $\bmu_k = \overline{\btheta}^{(0)}_{k+1,h}$ and 
for any $i \leq K$, 
\begin{align}
    y_i &= \overline{\sigma}_{i,h}^{-1}\big(\la \btheta^*_h, \bphi_{\overline{V}_{i,h+1}}(s_h^i, a_h^i, b_h^i)\ra + \ind\{\btheta^*_h \in \overline{\cC}^{(1)}_{i,h} \cap \overline{\cC}^{(2)}_{i,h} \}\big[\overline{V}_{i, h+1}(s_{h+1}^i)\notag\\
    &\qquad - \la\bphi_{\overline{V}_{i, h+1}}(s_h^i, a_h^i, b_h^i), \btheta^* \ra\big]\big) \notag \\
    & = \overline{\sigma}_{i,h}^{-1}\overline{V}_{i, h+1}(s_{h+1}^i),\notag
\end{align}
which implies the claim.
Therefore, by the definition  of $\cC_{k,h}^{(0)}$, we get that on $\event'$, $\btheta^*_h \in \cap_{k\leq K} \overline{\cC}^{(0)}_{k,h} \cap \overline{\cC}^{(1)}_{k,h}$. Moreover, $\PP(\event') \geq 1 - 3\delta/(2H)$. 
Finally, taking union bound over $h$
shows that with probability at least $1-3\delta/2$, for all $h\in [H]$,
\begin{align}
\btheta^*_h \in \cap_{k\le K} \overline{\cC}^{(1)}_{k,h} \cap \overline{\cC}^{(2)}_{k,h}
\label{eq:strongerconf}
\end{align}
To finish our proof, it is thus sufficient to show that on the event when \eqref{eq:strongerconf} holds, it also holds
that
\begin{align}
    \big|[\var^{\text{est}}_{k,h}\overline{V}_{k,h+1}](s_h^k, a_h^k,b_h^k) - [\var_h\overline{V}_{k,h+1}](s_h^k, a_h^k,b_h^k)\big| \leq \overline{\error}_{k,h}.\notag
\end{align}
However, by the definition of $\overline{\error}_{k,h}$, this is immediate from substituting \eqref{eq:concentrate:finite:1}, \eqref{eq:concentrate:finite:1.11} into 
Lemma~\ref{lemma: VB}.
\end{proof}

\subsection{Proof of Theorem \ref{mainthm}}
Let the event $\mathcal{E}$ denote the event when the conclusion of Lemma \ref{Thm:inset} holds. Then Lemma \ref{Thm:inset} suggests that $\mathbb{P}(\mathcal{E}) \geq 1-3\delta$. We introduce another two events in the following lemma.
\begin{lemma}\label{lemma:otherevent}
Denote events $\mathcal{E}_{1}$ and $\mathcal{E}_{2}$ as follows
\begin{align*}
\mathcal{E}_{1} &=  \Big\{\forall h'\in [H],  \sum_{k=1}^{K}\sum_{h=h'}^{H}\Big[[\mathbb{P}_{h}\overline{V}_{k,h+1}](s_{h}^{k},a_{h}^{k},b_{h}^{k}) - [\mathbb{P}_{h}\underline{V}_{k,h+1}](s_{h}^{k},a_{h}^{k},b_{h}^{k})\\
&\qquad- \overline{V}_{k,h+1}(s_{h+1}^{k}) + \underline{V}_{k,h+1}(s_{h+1}^{k})\Big] \leq  8H\sqrt{2T
\log(H/\delta)}\Big\}\\
\mathcal{E}_{2} &=  \Big\{\sum_{k=1}^{K}\sum_{h=1}^{H}\mathbb{V}_{h}V_{h+1}^{\mu^{k}}(s_{h}^{k},a_{h}^{k}, b_{h}^{k}) \leq 3(HT + H^{3}\log(1/\delta))\Big\}.
\end{align*}
Then we have $\mathbb{P}(\mathcal{E}_{1})\geq 1-\delta$ and $\mathbb{P}(\mathcal{E}_{2})\geq 1-\delta$.
\end{lemma}

We now present three lemmas based on  $\event, \event_1, \event_2$. The following lemma shows that $\overline{Q}$ and $\overline{V}$ provide the good UCB for the best response of the max-player and $\underline{Q}$ and $\underline{V}$ provide the good LCB for the best response of the min-player.
% \begin{lemma}[Lemma 4 in \citet{xie2020learning}]\label{lemma:2CCE}
% For each $(k,h,s)$, $\pi_{h}^{k}(s), \nu_{h}^{k}(s)$ satisfy that 
% \begin{align*}
% &\EE_{a\sim \pi_{h}^{k}(s), b\sim \nu_{h}^{k}(s)}\big[\overline{Q}_{k,h}(s,a,b)\big]\geq \EE_{b\sim \nu_{h}^{k}(s)}\big[\overline{Q}_{k,h}(s,a',b)\big] - 2\epsilon, \forall a'\in \cA\\
% &\EE_{a\sim \pi_{h}^{k}(s), b\sim \nu_{h}^{k}(s)}\big[\underline{Q}_{k,h}(s,a,b)\big]\leq \EE_{a\sim \pi_{h}^{k}(s)}\big[\underline{Q}_{k,h}(s,a,b')\big] - 2\epsilon, \forall b'\in \cA
% \end{align*}
% \end{lemma}

\begin{lemma}\label{Lemma:ULCB}
Suppose the event $\mathcal{E}$ hold, then we have for any s,a,b,k,h following inequalities hold,
\begin{align*}
\underline{Q}_{k,h}(s,a,b) - (H-h+1)\epsilon \leq Q_{h}^{\pi^{k}, *}(s,a,b) \leq Q_{h}^{*, \nu^{k}}(s,a,b) \leq \overline{Q}_{k,h}(s,a,b) + (H-h+1)\epsilon,
\end{align*}
and 
\begin{align*}
\underline{V}_{k,h}(s) -  (H-h+2)\epsilon \leq V_{h}^{\pi^{k},*}(s) \leq  V_{h}^{*,\nu^{k}}(s) \leq \overline{V}_{k,h}(s) + (H-h+2)\epsilon.  
\end{align*}
\end{lemma}

\begin{lemma}\label{Lemma:difference}
Suppose the events $\mathcal{E}\cap \mathcal{E}_{1}$ hold, then we have 
\begin{align*}
\sum_{k=1}^{K}[\overline{V}_{k,1}(s_{k,1}) - \underline{V}_{k,1}(s_{k,1})]&\leq 4\beta^{(0)}_{K}\sqrt{\sum_{k=1}^{K}\sum_{h=1}^{H}\overline{\sigma}_{k,h}^2 + \underline{\sigma}_{k,h}^2}\sqrt{2Hd\log(1+K/\lambda)}\notag\\
&\qquad + 8H\sqrt{2T\log(H/\delta)},
\end{align*}

\begin{align*}
\sum_{k=1}^{K}\sum_{h=1}^{H}\mathbb{P}_{h}[\overline{V}_{k,h+1} - \underline{V}_{k,h+1}](s_{h}^{k},a_{h}^{k},b_{h}^{k})&\leq  4\beta^{(0)}_{K}\sqrt{\sum_{k=1}^{K}\sum_{h=1}^{H}\overline{\sigma}_{k,h}^2+\underline{\sigma}_{k,h}^2}\sqrt{2H^{3}d\log(1+K/\lambda)}\\
&\qquad + 8H^{2}\sqrt{2T\log(H/\delta)},
\end{align*}
\end{lemma}

\begin{lemma}\label{Lemma:VARBOUND}
Suppose the events $\mathcal{E}\cap \mathcal{E}_{2}$ hold, then we have 
\begin{align*}
\sum_{k=1}^{K}\sum_{h=1}^{H}\overline{\sigma}_{k,h}^2 &\leq H^{2}T/d + 3(HT + H^{3}\log(1/\delta)) +  4H\sum_{k=1}^{K}\sum_{h=1}^{H}\mathbb{P}_{h}[\overline{V}_{k,h+1} - V_{h+1}^{\mu^{k}}]\\
&\qquad + 2\beta_{K}^{(2)}\sqrt{T}\sqrt{2dH\log(1+KH^{4}/(d\lambda))} + 7\beta_{K}^{(1)}H^{2}\sqrt{T}\sqrt{2dH\log(1+K/\lambda)}\\
\sum_{k=1}^{K}\sum_{h=1}^{H}\underline{\sigma}_{k,h}^2 &\leq H^{2}T/d + 3(HT + H^{3}\log(1/\delta)) +  4H\sum_{k=1}^{K}\sum_{h=1}^{H}\mathbb{P}_{h}[V_{h+1}^{\mu^{k}} - \underline{V}_{k,h+1}] \\
&\qquad + 2\beta_{K}^{(2)}\sqrt{T}\sqrt{2dH\log(1+KH^{4}/(d\lambda))} + 7\beta_{K}^{(1)}H^{2}\sqrt{T}\sqrt{2dH\log(1+K/\lambda)}\\
\end{align*}
\end{lemma}

With all these lemmas, we can now give the proof of Theorem~\ref{mainthm}.
\begin{proof}[Proof of Theorem \ref{mainthm}]
By definition of Regret we have that 
\begin{align}\label{Mainregert}
\text{Regret}(K) &= \sum_{k=1}^{K} V_{1}^{*, \nu^{k}}(s_{1}^{k})  - \sum_{k=1}^{K} V_{1}^{\pi^{k},*}(s_{1}^{k})\notag\\
&\leq  \sum_{k=1}^{K} \overline{V}_{k,1}(s_{k,1}) - \sum_{k=1}^{K} \underline{V}_{k,1}(s_{k,1}) + 4KH\epsilon\notag\\
&\leq  4\beta^{(0)}_{K}\sqrt{\sum_{k=1}^{K}\sum_{h=1}^{H}\overline{\sigma}_{k,h}^2 + \underline{\sigma}_{k,h}^2}\sqrt{2Hd\log(1+K/\lambda)}\notag + 8H\sqrt{2T\log(H/\delta)}+4KH\epsilon\\
&= \tilde O\bigg(d\sqrt{H}\sqrt{\sum_{k=1}^{K}\sum_{h=1}^{H}\overline{\sigma}_{k,h}^2 + \underline{\sigma}_{k,h}^2} + H\sqrt{T}\bigg),
\end{align}
where the first inequality is by Lemma \ref{Lemma:ULCB}, the second inequality is by the bound of accumulated difference between the UCB and LCB in Lemma \ref{Lemma:difference}, the last inequality is due to  $\epsilon = O(H/\sqrt{T})$, $\lambda = 1/B^{2}$ and the choice of  $\beta^{(0)}_{K} = \tilde O(\sqrt{d})$ in Lemma \ref{Thm:inset}.

Now we bound $\sum_{k=1}^{K}\sum_{h=1}^{H}\overline{\sigma}_{k,h}^2 + \underline{\sigma}_{k,h}^2$, 
\begin{align}\label{aggregatebound}
&\sum_{k=1}^{K}\sum_{h=1}^{H}\overline{\sigma}_{k,h}^2 + \underline{\sigma}_{k,h}^2\\
&\leq 2H^{2}T/d + 6(HT + H^{3}\log(1/\delta))  + 4H\sum_{k=1}^{K}\sum_{h=1}^{H}\mathbb{P}_{h}[\overline{V}_{k,h+1} - \underline{V}_{k,h+1}]\notag\\
&\qquad + 4\beta_{K}^{(2)}\sqrt{T}\sqrt{2dH\log(1+KH^{4}/(d\lambda))} + 14\beta_{K}^{(1)}H^{2}\sqrt{T}\sqrt{2dH\log(1+K/\lambda)}\notag\\
&\leq 2H^{2}T/d + 6(HT + H^{3}\log(1/\delta)) \notag\\
&\qquad+ 4H\bigg(4\beta^{(0)}_{K}\sqrt{\sum_{k=1}^{K}\sum_{h=1}^{H}\overline{\sigma}_{k,h}^2+\underline{\sigma}_{k,h}^2}\sqrt{2H^{3}d\log(1+K/\lambda)}+ 8H^{2}\sqrt{2T\log(H/\delta)}\bigg)\notag\\
&\qquad + 4\beta_{K}^{(2)}\sqrt{T}\sqrt{2dH\log(1+KH^{4}/(d\lambda))} + 14\beta_{K}^{(1)}H^{2}\sqrt{T}\sqrt{2dH\log(1+K/\lambda)}\notag\\
&= \tilde O\bigg(\sqrt{\sum_{k=1}^{K}\sum_{h=1}^{H}\overline{\sigma}_{k,h}^2+\underline{\sigma}_{k,h}^2}\sqrt{d^{2}H^{5}} + H^{2}T/d + TH + \sqrt{T}d^{1.5}H^{2.5} + H^{3}\sqrt{T}\bigg)
\end{align}
where the first inequality is by Lemma \ref{Lemma:VARBOUND}, the second inequality is by Lemma \ref{Lemma:difference} and the last inequality is due to the choice of  $\beta_{K}^{(0)} = \tilde{O}(\sqrt{d})$ in Lemma \ref{Thm:inset}, $\lambda = 1/B^{2}$,
\begin{align*}
\beta_{K}^{(1)} &= 16\sqrt{dH^{4}\log(1 + kH^{4}/d\lambda)\log(8k^{2}H/\delta)} + 8H^{2}\log(8k^{2}H/\delta) + \sqrt{\lambda}B =\tilde O(dH^{2})\\
\beta_{K}^{(2)} &= 16d\sqrt{\log(1 + k/\lambda)\log(8k^{2}H/\delta)} + 8\sqrt{d}\log(8k^{2}H/\delta) + \sqrt{\lambda}B =\tilde O(d)
.
\end{align*}
Therefore by the fact that $x\leq a\sqrt{x} + b \Rightarrow x\leq 2a^{2}+b$, \eqref{aggregatebound} suggests that 
\begin{align}\label{aggregatebound2}
\sum_{k=1}^{K}\sum_{h=1}^{H}\overline{\sigma}_{k,h}^2 + \underline{\sigma}_{k,h}^2 &= \tilde{O}(d^{2}H^{5} + H^{2}T/d + TH + \sqrt{T}d^{1.5}H^{2.5} + H^{3}\sqrt{T})\notag\\
&= \tilde{O}(d^{2}H^{5} + d^{4}H^{3} + TH + H^{2}T/d),
\end{align}
where the inequality holds by $\sqrt{T}d^{1.5}H^{2.5}\leq (TH^{2}/4d+ d^{4}H^{3})/2$ and $ H^{3}\sqrt{T} \leq (d^{2}H^{5}+H^{2}T/d)/2$. Plugging \eqref{aggregatebound2} into \eqref{Mainregert} we have 
\begin{align*}
\text{Regret}(M_{\btheta^{*}}, K) &=\tilde{\cO}(\sqrt{d^{2}H^{2} + dH^{3}}\sqrt{T} + d^{2}H^{3} + d^{3}H^{2}),
\end{align*}
which finishes the proof.
\end{proof}

\subsection{Proof of Theorem \ref{theorem:lowerbound}}
\begin{proof}[Proof of Theorem \ref{theorem:lowerbound}]
For any algorithm, we need to construct a hard-to-learn episodic,B-bounded linear mixture Markov game. We make the min-player dummy: the action of the min-player won't affect the transition ability or reward function. So there exists $\tilde{\PP}_{h}(\cdot|\cdot,\cdot)$ and $\tilde{r_{h}}(\cdot,\cdot)$ such that for any state-action-action-state pair $s',a,b,s$ we have that $\PP_{h}(s'|s,a,b)=\tilde{\PP}_{h}(s'|s,a)$ and $r_{h}(s,a,b) = \tilde{r_{h}}(s,a)$. Thus we can get a new MDP $\tilde{M}(\cS, \cA_{\max}, H, \{\tilde{r}_{h}\}, \{\tilde{\PP}_{h}\})$. We further have $V_{h}^{\pi,*}(s) =\tilde{V}_{h}^{\pi}(s)$ and $V_{h}^{*,\nu}(s) =\tilde{V}_{h}^{*}(s)$. The regret of two-player game can be reduced to the standard regret for single agent reinforcement learning setting. In particular,
\begin{align*}
\mathrm{Regret}(M_{\btheta^{*}},K) &= \sum_{k=1}^{K} V_{1}^{*, \nu^{k}}(s_{1}^{k})  - \sum_{k=1}^{K} V_{1}^{\pi^{k},*}(s_{1}^{k})\\
& = \sum_{k=1}^{K} \tilde{V}_{1}^{*}(s_{1}^{k})  - \sum_{k=1}^{K} \tilde{V}_{1}^{\pi^{k}}(s_{1}^{k}).
\end{align*}
Notice that $\tilde{r}_{h} \in [-1,1]$ rather than $[0,1]$, we can shift the reward by $(1+\tilde{r}_{h})/2$ to make it standard if necessary. Now
recall the Theorem 5.6 in \citet{zhou2020nearly}],
    there exists an episodic, $B$-bounded linear mixture MDP $\tilde{M}(\cS, \cA_{\max}, H, \{\tilde{r}_{h}\}, \{\tilde{\PP}_{h}\})$ with feature $\tilde{\bphi}(\cdot,\cdot)$ parameterized by $\bTheta=(\btheta_1,\dots, \btheta_H)$ such that the expected regret is lower bounded as follows: 
% \todoc{I guess the rewards should be normalized..}
\begin{align}
     \EE_{\bTheta}\text{Regret}\big( \tilde{M}_{\bTheta}, K\big) \geq \Omega\big(dH\sqrt{T}\big),\notag
\end{align}
where $T= KH$ and $\EE_{\bTheta}$ denotes the expectation over 
the probability distribution generated by the interconnection of the algorithm and the MDP.

Now we only need to extend the MDP feature $\tilde{\bphi}(\cdot| \cdot, \cdot)$ to the Markov game feature $\bphi(\cdot| \cdot, \cdot, \cdot)$. In particular, we set
\begin{align*}
\bphi(s'|s,a,b) = \tilde{\bphi}(s'|s,a), \forall s' \in \cS,s \in \cS,a \in \cA_{\max},b \in \cA_{\min}, 
\end{align*}
then we know that $\bphi(\cdot|\cdot, \cdot,\cdot)$ satisfies \eqref{mapping} because by the definition of linear mixture MDP in \citet{zhou2020nearly}, we know that $\tilde{\bphi}(\cdot|\cdot,\cdot)$ satisfies for any bounded function $V:\cS \rightarrow [0,1]$,
\begin{align*}
\|\tilde{\bphi}_{V}(s,a)\|_{2} \leq 1,
\end{align*}
where $\tilde\bphi_V(s,a) = \sum_{s' \in \cS}\tilde\bphi(s'|s,a)V(s').$
\end{proof}

\section{Proof of Lemmas in Appendix \ref{sec:AppendixA}}\label{sec:AppendixB}
\subsection{Proof of Lemma \ref{lemma: VB}}
\begin{proof}[Proof of Lemma \ref{lemma: VB}]
For simplicity, we only prove the results for the max-player.

By the triangle inequality we have that
\begin{align}
&|\mathbb{V}^{\text{est}}\overline{V}_{k,h+1}(s_{h}^{k},a_{h}^{k},b_{h}^{k}) -   \mathbb{V}\overline{V}_{k,h+1}(s_{h}^{k},a_{h}^{k},b_{h}^{k})|\notag\\
&\leq \underbrace{\Big|\langle\bphi_{\overline{V}^{2}_{k,h+1}}(s_{h}^{k},a_{h}^{k},b_{h}^{k}), \btheta^{*}_{h}\rangle  - \big[\langle\bphi_{\overline{V}^{2}_{k,h+1}}(s_{h}^{k},a_{h}^{k},b_{h}^{k}), \overline{\btheta}^{(1)}_{k,h}\rangle\big]_{[0,H^{2}]} \Big|}_{I_{1}}\notag\\
&\qquad + \underbrace{\Big| (\langle\bphi_{\overline{V}_{k,h+1}}(s_{h}^{k},a_{h}^{k},b_{h}^{k}), \btheta^{*}_{h}\rangle)^{2}-  
\big[\langle\bphi_{\overline{V}_{k,h+1}}(s_{h}^{k},a_{h}^{k},b_{h}^{k}), \overline{\btheta}^{(0)}_{k,h}\rangle\big]^{2}_{[-H,H]}\Big|}_{I_{2}} \label{eq:inSet-1}.
\end{align}
We first bound $I_{1}$. Because $\langle\bphi_{\overline{V}^{2}_{k,h+1}}(s_{h}^{k},a_{h}^{k},b_{h}^{k}), \btheta_{h}^{*}\rangle \in [0,H^{2}]$, we have that 
\begin{align*}
I_{1} &\leq \Big| \langle\bphi_{\overline{V}^{2}_{k,h+1}}(s_{h}^{k},a_{h}^{k},b_{h}^{k}), \btheta^{*}_{h}\rangle - \langle\bphi_{\overline{V}^{2}_{k,h+1}}(s_{h}^{k},a_{h}^{k},b_{h}^{k}), \overline{\btheta}^{(1)}_{k,h}\rangle \Big|\\
&\leq \Big\|\big[\overline{\bSigma}^{(1)}_{k,h}\big]^{-1/2}\bphi_{\overline{V}^{2}_{k,h+1}}(s_{h}^{k},a_{h}^{k},b_{h}^{k})\Big\|_{2}\Big\|\big[\overline{\bSigma}^{(1)}_{k,h}\big]^{1/2}(\overline{\btheta}^{(1)}_{k,h} - \btheta^{*}_{h})\Big\|_{2},
\end{align*}
where the first inequality is by the property of projection, the second inequality holds due to Cauchy-Schwarz.
We also have that $I_{1} \leq H^{2}$ since both terms in $I_{1}$ belongs to the interval $[0, H^2]$, so we have that \begin{align}
I_{1} &\leq \min\Big\{H^{2}, \Big\|\big[\overline{\bSigma}^{(1)}_{k,h}\big]^{-1/2}\bphi_{\overline{V}^{2}_{k,h+1}}(s_{h}^{k},a_{h}^{k},b_{h}^{k})\Big\|_{2}\Big\|\big[\overline{\bSigma}^{(1)}_{k,h}\big]^{1/2}(\overline{\btheta}^{(1)}_{k,h} - \btheta^{*}_{h})\Big\|_{2}\Big\}\label{eq:inSet-2},
\end{align}
For the term $I_{2}$,
\begin{align*}
I_{2} &= \Big| \langle\bphi_{\overline{V}_{k,h+1}}(s_{h}^{k},a_{h}^{k},b_{h}^{k}), \btheta^{*}_{h}\rangle-  
\big[\langle\bphi_{\overline{V}_{k,h+1}}(s_{h}^{k},a_{h}^{k},b_{h}^{k}), \overline{\btheta}^{(0)}_{k,h}\rangle\big]_{[-H,H]}\Big|\\
&\qquad \cdot\Big| \langle\bphi_{\overline{V}_{k,h+1}}(s_{h}^{k},a_{h}^{k},b_{h}^{k}), \btheta^{*}_{h}\rangle +  
\big[\langle\bphi_{\overline{V}_{k,h+1}}(s_{h}^{k},a_{h}^{k},b_{h}^{k}), \overline{\btheta}^{(0)}_{k,h}\rangle\big]_{[-H,H]}\Big|\\
&\leq 2H\Big| \langle\bphi_{\overline{V}_{k,h+1}}(s_{h}^{k},a_{h}^{k},b_{h}^{k}), \btheta^{*}_{h}\rangle-  
\langle\bphi_{\overline{V}_{k,h+1}}(s_{h}^{k},a_{h}^{k},b_{h}^{k}), \overline{\btheta}^{(0)}_{k,h}\rangle\Big|\\
&\leq 2H\Big\|\big[\overline{\bSigma}^{(0)}_{k,h}\big]^{-1/2}\bphi_{\overline{V}^{2}_{k,h+1}}(s_{h}^{k},a_{h}^{k},b_{h}^{k})\Big\|_{2}\Big\|\big[\overline{\bSigma}^{(0)}_{k,h}\big]^{1/2}(\overline{\btheta}^{(0)}_{k,h} - \btheta^{*}_{h})\Big\|_{2},
\end{align*}
where the first inequality holds since both terms in this line lies in $[-H,H]$, the second inequality holds since the Cauchy-Schwarz inequality. We also have that $I_{2} \leq H^{2}$, so we have that
\begin{align}
I_{2} \leq \min\Big\{H^{2}, 2H\Big\|\bSigma^{(0)-1/2}_{k,h}\bphi_{\overline{V}^{2}_{k,h+1}}(s_{h}^{k},a_{h}^{k},b_{h}^{k})\Big\|_{2}\Big\|\big[\overline{\bSigma}^{(0)}_{k,h}\big]^{1/2}(\overline{\btheta}^{(0)}_{k,h} - \btheta^{*}_{h})\Big\|_{2}\Big\}\label{eq:inSet-3}. 
\end{align}
Plugging \eqref{eq:inSet-3} and \eqref{eq:inSet-2} into \eqref{eq:inSet-1} gets 
\begin{align*}
&|\mathbb{V}^{\text{est}}\overline{V}_{k,h+1}(s_{h}^{k},a_{h}^{k},b_{h}^{k}) -   \mathbb{V}\overline{V}_{k,h+1}(s_{h}^{k},a_{h}^{k},b_{h}^{k})|\\
&\leq \min\Big\{H^{2}, \Big\|\big[\overline{\bSigma}^{(1)}_{k,h}\big]^{-1/2}\bphi_{\overline{V}^{2}_{k,h+1}}(s_{h}^{k},a_{h}^{k},b_{h}^{k})\Big\|_{2}\Big\|\big[\overline{\bSigma}^{(1)}_{k,h}\big]^{1/2}(\overline{\btheta}^{(1)}_{k,h} - \btheta^{*}_{h})\Big\|_{2}\Big\}\\
&\qquad + \min\Big\{H^{2}, 2H\Big\|\big[\overline{\bSigma}^{(0)}_{k,h}\big]^{-1/2}\bphi_{\overline{V}^{2}_{k,h+1}}(s_{h}^{k},a_{h}^{k},b_{h}^{k})\Big\|_{2}\Big\|\big[\overline{\bSigma}^{(0)}_{k,h}\big]^{1/2}(\overline{\btheta}^{(0)}_{k,h} - \btheta^{*}_{h})\Big\|_{2}\Big\}.
\end{align*}
\end{proof}

\subsection{Proof of Lemma \ref{lemma:otherevent}}
We first present the Azuma-Hoeffding inequality: \begin{lemma}[Azuma-Hoeffding inequality, \citealt{azuma1967weighted}]\label{lemma:azuma} 
Let $M>0$ be a constant.
Let $\{x_i\}_{i=1}^n$ be a martingale difference sequence with respect to a filtration $\{\cG_{i}\}_i$ 
($\EE[x_i|\cG_i]=0$ a.s. and $x_i$ is $\cG_{i+1}$-measurable) such that
for all $i\in [n]$, $|x_i| \leq M$ 
% \todoc{I am suspecting we need a variant where what is assumed is that $x_i$ takes values in a random interval $[A_i,A_i+M]$, where $A_i$ is $\cG_i$ measurable. Or we lose some constant factor of two later, though we probably do not care about this.}
holds almost surely. 
Then, for any $0<\delta<1$, with probability at least $1-\delta$, we have 
\begin{align}
    \sum_{i=1}^n x_i\leq M\sqrt{2n \log (1/\delta)}.\notag
\end{align} 
\end{lemma}
\begin{proof}[Proof of Lemma \ref{lemma:otherevent}]
To prove $\PP(\event_{1})\geq 1- \delta$, we apply the Azuma-Hoeffding inequality (Lemma \ref{lemma:azuma}). Fix $h'\in H$, set $x_{k,h} = [\mathbb{P}_{h}\overline{V}_{k,h+1}](s_{h}^{k},a_{h}^{k},b_{h}^{k}) - [\mathbb{P}_{h}\underline{V}_{k,h+1}](s_{h}^{k},a_{h}^{k},b_{h}^{k}) - [\overline{V}_{k,h+1}(s_{h+1}^{k}) - \underline{V}_{k,h+1}(s_{h+1}^{k})]$. $x_{1,h'},\ldots,x_{1,H},x_{2,h'},\ldots,x_{2,H},\ldots,x_{K,h'},\ldots,x_{K,H}$ forms a martingale difference sequence of which the absolute value is bounded by 8H and length no greater than $T=KH$. Thus with probability at least $1-\delta/H$, we have
\begin{align*}
&\sum_{k=1}^{K}\sum_{h=h'}^{H}\Big[[\mathbb{P}_{h}\overline{V}_{k,h+1}](s_{h}^{k},a_{h}^{k},b_{h}^{k}) - [\mathbb{P}_{h}\underline{V}_{k,h+1}](s_{h}^{k},a_{h}^{k},b_{h}^{k}) - \overline{V}_{k,h+1}(s_{h+1}^{k}) + \underline{V}_{k,h+1}(s_{h+1}^{k})\Big]\\
&\leq  8H\sqrt{2T
\log(H/\delta)}.
\end{align*}
Take union bound for $h'\in [H]$, we get $\PP(\event_{1})\geq 1-\delta$.

$\mathbb{P}(\mathcal{E}_{2})\geq 1-\delta$ holds due to the Lemma C.5 in \citet{jin2018q} or Lemma 8 in \citet{azar2017minimax}.
\end{proof}
\subsection{Proof of Lemma \ref{Lemma:ULCB}}

Following Lemma directly from the definition of $\epsilon$-CCE,
\begin{lemma}\label{lemma:2CCE}
For each $(k,h,s)$, $\mu_{h}^{k}(\cdot,\cdot|s), \pi_{h}^{k}(\cdot|s), \nu_{h}^{k}(\cdot|s)$ satisfy that 
\begin{align*}
&\EE_{(a,b)\sim \mu_{h}^{k}(\cdot,\cdot|s)}\big[\overline{Q}_{k,h}(s,a,b)\big]\geq \EE_{b\sim \nu_{h}^{k}(s)}\big[\overline{Q}_{k,h}(s,a',b)\big] - \epsilon, \forall a'\in \cA_{\max}\\
&\EE_{(a,b)\sim \mu_{h}^{k}(\cdot,\cdot|s)}\big[\underline{Q}_{k,h}(s,a,b)\big]\leq \EE_{a\sim \pi_{h}^{k}(s)}\big[\underline{Q}_{k,h}(s,a,b')\big] - \epsilon, \forall b'\in \cA_{\min}
\end{align*}
\end{lemma}
\begin{proof}[Proof of Lemma \ref{Lemma:ULCB}]
For simplicity, we only prove the following UCB by induction,
\begin{align}
Q_{h}^{*, \nu^{k}}(s,a,b) \leq \overline{Q}_{k,h}(s,a,b) + (H-h+1)\epsilon, V_{h}^{*,\nu^{k}}(s) \leq \overline{V}_{k,h}(s) + (H-h+2)\epsilon. \label{eq: UCB}
\end{align}
The base case $h = H + 1$ holds trivially since the terminal cost is zero. Now we assume that the bounds \eqref{eq: UCB} holds for step $h+1$. That is, 
\begin{align}
Q_{h+1}^{*, \nu^{k}}(s,a,b) \leq \overline{Q}_{k,h+1}(s,a,b) + (H-h)\epsilon, V_{h+1}^{*,\nu^{k}}(s) \leq \overline{V}_{k,h+1}(s) + (H-h+1)\epsilon. \label{eq: UCB2}
\end{align}
If $\overline{Q}_{k,h}(s,a,b) \geq H$, then it is obvious to have $Q_{h}^{*, \nu^{k}}(s,a,b) \leq \overline{Q}_{k,h}(s,a,b) + (H-h)\epsilon$, otherwise we have that
\begin{align}
&\overline{Q}_{k,h}(s,a,b) -  Q_{h}^{*, \nu^{k}}(s,a,b)\notag\\
&=  \langle \overline{\btheta}^{(0)}_{k,h}, \bphi_{\overline{V}_{k,h+1}} \rangle + \beta^{(0)}_{k}\Big\|\big[\overline{\bSigma}_{k,h}^{(0)}\big]^{-1/2}\bphi_{\overline{V}_{k,h+1}}\Big\|_{2} - \langle \btheta_{h}^{*}, \bphi_{\overline{V}_{k,h+1}} \rangle \notag\\
&\qquad+ \mathbb{P}_{h}\overline{V}_{k,h+1}(s,a,b) - \mathbb{P}_{h}V_{h+1}^{*,\nu^{k}}(s)\notag\\
&\geq \beta_{k}^{(0)}\Big\|\big[\overline{\bSigma}_{k,h}^{(0)}\big]^{-1/2}\bphi_{\overline{V}_{k,h+1}}\Big\|_{2} - \Big\|\big[\overline{\bSigma}_{k,h}^{(0)}\big]^{1/2}(\overline{\btheta}^{(0)}_{k,h} - \btheta_{h}^{*})\Big\|_{2}\Big\|\big[\overline{\bSigma}_{k,h}^{(0)}\big]^{-1/2}\bphi_{\overline{V}_{k,h+1}}\Big\|_{2}\notag\\
&\qquad+ \mathbb{P}_{h}\overline{V}_{k,h+1}(s) - \mathbb{P}_{h}V_{h+1}^{*,\nu^{k}}(s)\notag\\
&\geq \mathbb{P}_{h}\overline{V}_{k,h+1}(s) - \mathbb{P}_{h}V_{h+1}^{*,\nu^{k}}(s)\notag\\
&\geq -(H-h+1)\epsilon\label{eq:Q},
\end{align}
where the first inequality holds due to Cauchy-Schwarz inequality, the second inequality holds since
the assumption that $\theta^{*}_{h} \in \overline{\cC}^{(0)}_{k,h}$ on event $\mathcal{E}$, the third inequality holds by the induction assumption.
Finally, let $\text{br}(\nu_{h}^{k}(\cdot|s))$ denote the best response to $\nu_{h}^{k}(\cdot|s)$ with respect to $Q_{h}^{*, \nu^{k}}(s, \cdot, \cdot)$ such that
\begin{align*}
\text{br}(\nu_{h}^{k}(\cdot|s)) &= \argmax_{\sigma \in \Delta_{\cA_{\max}}}\EE_{a\sim \sigma,b\sim \nu^{k}(\cdot|s)}Q_{h}^{*,\nu^{k}}(s, a, b).    
\end{align*}
Then we have that 
\begin{align*}
\overline{V}_{k,h}(s)&= \EE_{(a,b)\sim \mu_{h}^{k}(\cdot,\cdot|s)}\big[\overline{Q}_{k,h}(s,a,b)\big]\\
&\geq \EE_{a'\sim \text{br}(v_{h}^{k}(\cdot|s)), b\sim \nu_{h}^{k}(\cdot|s)}\big[\overline{Q}_{k,h}(s,a',b)\big] - \epsilon\\
&\geq \EE_{a'\sim \text{br}(v_{h}^{k}(\cdot|s)), b\sim \nu_{h}^{k}(\cdot|s)}\big[Q_{h}^{*, \nu^{k}}(s,a',b)\big] - (H - h +2)\epsilon\\
&= V_{h}^{*, \nu^{k}}(s) - (H - h +2)\epsilon,
\end{align*}
where the the first equality is by the property of $\epsilon$-CCE in Lemma \ref{lemma:2CCE}, the second inequality is by \eqref{eq:Q}, the last inequality is due to the Bellman equation.Therefore,
our proof ends.
\end{proof}
\subsection{Proof of Lemma \ref{Lemma:difference}}

\begin{proof}[Proof of Lemma \ref{Lemma:difference}]
\begin{align*}
&\overline{V}_{k,h}(s_{h}^{k}) - \underline{V}_{k,h}(s_{h}^{k})\\
&= \langle \overline{\btheta}^{(0)}_{k,h}, \bphi_{\overline{V}_{k,h+1}} \rangle + \beta^{(0)}_{k}\Big\|\big[\overline{\bSigma}_{k,h}^{(0)}\big]^{-1/2}\bphi_{\overline{V}_{k,h+1}}\Big\|_{2} -\langle\underline{\btheta}^{(0)}_{k,h}, \bphi_{\underline{V}_{k,h+1}} \rangle + \beta^{(0)}_{k}\Big\|\big[\underline{\bSigma}_{k,h}^{(0)}\big]^{-1/2}\bphi_{\underline{V}_{k,h+1}}\Big\|_{2}\\
&= \langle \btheta^{*}_{h}, \bphi_{\overline{V}_{k,h+1}} \rangle + \langle \overline{\btheta}^{(0)}_{k,h} - \btheta^{*}_{h}, \bphi_{\overline{V}_{k,h+1}} \rangle + \beta^{(0)}_{k}\Big\|\big[\overline{\bSigma}_{k,h}^{(0)}\big]^{-1/2}\bphi_{\overline{V}_{k,h+1}}\Big\|_{2} \\
&\qquad -\langle\btheta^{*}_{h}, \bphi_{\underline{V}_{k,h+1}} \rangle - \langle\underline{\btheta}^{(0)}_{k,h} - \btheta^{*}_{h}, \bphi_{\underline{V}_{k,h+1}} \rangle + \beta^{(0)}_{k}\Big\|\big[\underline{\bSigma}_{k,h}^{(0)}\big]^{-1/2}\bphi_{\underline{V}_{k,h+1}}\Big\|_{2}\\
&\leq \langle \btheta^{*}_{h}, \bphi_{\overline{V}_{k,h+1}} \rangle + \Big\|\big[\overline{\bSigma}_{k,h}^{(0)}\big]^{1/2}(\overline{\btheta}^{(0)}_{k,h} - \btheta^{*}_{h})\Big\|_{2}\Big\|\big[\overline{\bSigma}_{k,h}^{(0)}\big]^{-1/2}\bphi_{\overline{V}_{k,h+1}}\Big\|_{2} + \beta^{(0)}_{k}\Big\|\big[\overline{\bSigma}_{k,h}^{(0)}\big]^{-1/2}\bphi_{\overline{V}_{k,h+1}}\Big\|_{2} \\
&\qquad -\langle\btheta^{*}_{h}, \bphi_{\underline{V}_{k,h+1}} \rangle + \Big\|\big[\underline{\bSigma}_{k,h}^{(0)}\big]^{1/2}(\underline{\btheta}^{(0)}_{k,h} - \btheta^{*}_{h})\Big\|_{2}\Big\|\big[\underline{\bSigma}_{k,h}^{(0)}\big]^{-1/2}\bphi_{\underline{V}_{k,h+1}}\Big\|_{2} \\
&\qquad+ \beta^{(0)}_{k}\Big\|\big[\underline{\bSigma}_{k,h}^{(0)}\big]^{-1/2}\bphi_{\underline{V}_{k,h+1}}\Big\|_{2}\\
&\leq [\mathbb{P}_{h}\overline{V}_{k,h+1}](s_{h}^{k},a_{h}^{k},b_{h}^{k}) + 2\beta^{(0)}_{k}\Big\|\big[\overline{\bSigma}_{k,h}^{(0)}\big]^{-1/2}\bphi_{\overline{V}_{k,h+1}}\Big\|_{2} \\
&\qquad- [\mathbb{P}_{h}\underline{V}_{k,h+1}](s_{h}^{k},a_{h}^{k},b_{h}^{k})+ 2\beta^{(0)}_k\Big\|\big[\underline{\bSigma}_{k,h}^{(0)}\big]^{-1/2}\bphi_{\underline{V}_{k,h+1}}\Big\|_{2},
\end{align*}
where the first equation is by the definition of $\overline{V}_{k,h}(s_{h}^{k}),\underline{V}_{k,h}(s_{h}^{k})$ and the second inequality is due to Cauchy-Schwarz inequality, the last inequality is by $\btheta_{h}^{*} \in \overline{\cC}_{k,h}^{(0)} \cap \underline{\cC}_{k,h}^{(0)}$ on the event $\mathcal{E}$.

Meanwhile, since $\overline{V}_{k,h}(s_{h}^{k}) - \underline{V}_{k,h}(s_{h}^{k}) \leq 2H$, we have that 
\begin{align*}
\overline{V}_{k,h}(s_{h}^{k}) - \underline{V}_{k,h}(s_{h}^{k})
&\leq \min\Big\{2H, [\mathbb{P}_{h}\overline{V}_{k,h+1}](s_{h}^{k},a_{h}^{k},b_{h}^{k}) + 2\beta^{(0)}_{k}\Big\|\big[\overline{\bSigma}_{k,h}^{(0)}\big]^{-1/2}\bphi_{\overline{V}_{k,h+1}}\Big\|_{2}\\
&\qquad - [\mathbb{P}_{h}\underline{V}_{k,h+1}](s_{h}^{k},a_{h}^{k},b_{h}^{k}) + 2\beta^{(0)}_{k}\Big\|\big[\underline{\bSigma}_{k,h}^{(0)}\big]^{-1/2}\bphi_{\underline{V}_{k,h+1}}\Big\|_{2}\Big\}\\
&\leq \min\Big\{4H, 2\beta^{(0)}_{k}\Big\|\big[\overline{\bSigma}_{k,h}^{(0)}\big]^{-1/2}\bphi_{\overline{V}_{k,h+1}}\Big\|_{2}+2\beta^{(0)}_{k}\Big\|\big[\underline{\bSigma}_{k,h}^{(0)}\big]^{-1/2}\bphi_{\underline{V}_{k,h+1}}\Big\|_{2}\Big\}\\
&\qquad + [\mathbb{P}_{h}\overline{V}_{k,h+1}](s_{h}^{k},a_{h}^{k},b_{h}^{k}) - [\mathbb{P}_{h}\underline{V}_{k,h+1}](s_{h}^{k},a_{h}^{k},b_{h}^{k})\\
&\leq \min\Big\{4H, 2\beta^{(0)}_{k}\Big\|\big[\overline{\bSigma}_{k,h}^{(0)}\big]^{-1/2}\bphi_{\overline{V}_{k,h+1}}\Big\|_{2}\Big\}\\
&\qquad +\min\Big\{4H,2\beta^{(0)}_{k}\Big\|\big[\underline{\bSigma}_{k,h}^{(0)}\big]^{-1/2}\bphi_{\underline{V}_{k,h+1}}\Big\|_{2}\Big\} \\
&\qquad+ [\mathbb{P}_{h}\overline{V}_{k,h+1}](s_{h}^{k},a_{h}^{k},b_{h}^{k}) - [\mathbb{P}_{h}\underline{V}_{k,h+1}](s_{h}^{k},a_{h}^{k},b_{h}^{k})\\
&\leq 2\beta^{(0)}_{k}\overline{\sigma}_{k,h}\min\Big\{1, \Big\|\big[\overline{\bSigma}_{k,h}^{(0)}\big]^{-1/2}\bphi_{\overline{V}_{k,h+1}}/\overline{\sigma}_{k,h}\Big\|_{2}\Big\}\\
&\qquad +2\beta^{(0)}_{k}\underline{\sigma}_{k,h}\min\Big\{1,\Big\|\big[\underline{\bSigma}_{k,h}^{(0)}\big]^{-1/2}\bphi_{\underline{V}_{k,h+1}}/\underline{\sigma}_{k,h}\Big\|_{2}\Big\} \\
&\qquad+ [\mathbb{P}_{h}\overline{V}_{k,h+1}](s_{h}^{k},a_{h}^{k},b_{h}^{k}) - [\mathbb{P}_{h}\underline{V}_{k,h+1}](s_{h}^{k},a_{h}^{k},b_{h}^{k}),
\end{align*}
where the second inequality holds because $[\mathbb{P}_{h}\overline{V}_{k,h+1}](s_{h}^{k},a_{h}^{k},b_{h}^{k}) - [\mathbb{P}_{h}\underline{V}_{k,h+1}](s_{h}^{k},a_{h}^{k},b_{h}^{k})\geq -2H$, the last inequality holds since $\beta^{(0)}_{k}\overline{\sigma}_{k,h} \geq 2H$, $\beta^{(0)}_{k}\underline{\sigma}_{k,h} \geq 2H$. Subtracting $\overline{V}_{k,h+1}(s_{h+1}^{k}) - \underline{V}_{k,h+1}(s_{h+1}^{k})$ from the both side, we can further get,
\begin{align}\label{single}
&\overline{V}_{k,h}(s_{h}^{k}) - \underline{V}_{k,h}(s_{h}^{k}) - [\overline{V}_{k,h+1}(s_{h}^{k}) - \underline{V}_{k,h+1}(s_{h}^{k})]\notag\\
&\leq 2\beta^{(0)}_{k}\overline{\sigma}_{k,h}\min\Big\{1, \Big\|\big[\overline{\bSigma}_{k,h}^{(0)}\big]^{-1/2}\bphi_{\overline{V}_{k,h+1}}/\overline{\sigma}_{k,h}\Big\|_{2}\Big\}\notag\\
&\qquad+ 2\beta^{(0)}_{k}\underline{\sigma}_{k,h}\min\Big\{1,\Big\|\big[\underline{\bSigma}_{k,h}^{(0)}\big]^{-1/2}\bphi_{\underline{V}_{k,h+1}}/\underline{\sigma}_{k,h}\Big\|_{2}\Big\}\notag \\
&\qquad+ [\mathbb{P}_{h}\overline{V}_{k,h+1}](s_{h}^{k},a_{h}^{k},b_{h}^{k}) - [\mathbb{P}_{h}\underline{V}_{k,h+1}](s_{h}^{k},a_{h}^{k},b_{h}^{k})\notag\\
&\qquad- [\overline{V}_{k,h+1}(s_{h+1}^{k}) - \underline{V}_{k,h+1}(s_{h+1}^{k})],    
\end{align}
Taking summation of \eqref{single} from $k=1 \ldots K$ and $h = h' \ldots H$, we have following inequality holds
\begin{align}\label{summation1}
&\sum_{k=1}^{K}[\overline{V}_{k,h'}(s_{h'}^{k}) - \underline{V}_{k,h'}(s_{h'}^{k})]\notag\\
&\leq 2\beta^{(0)}_{k}\sum_{k=1}^{K}\sum_{h=h'}^{H}\overline{\sigma}_{k,h}\min\Big\{1, \Big\|\big[\overline{\bSigma}_{k,h}^{(0)}\big]^{-1/2}\bphi_{\overline{V}_{k,h+1}}/\overline{\sigma}_{k,h}\Big\|_{2}\Big\}\notag\\
&\qquad+ 2\beta^{(0)}_{k}\sum_{k=1}^{K}\sum_{h=h'}^{H}\underline{\sigma}_{k,h}\min\Big\{1,\Big\|\big[\underline{\bSigma}_{k,h}^{(0)}\big]^{-1/2}\bphi_{\underline{V}_{k,h+1}}/\underline{\sigma}_{k,h}\Big\|_{2}\Big\}\notag \\
&\qquad+ \sum_{k=1}^{K}\sum_{h=h'}^{H}\Big[[\mathbb{P}_{h}\overline{V}_{k,h+1}](s_{h}^{k},a_{h}^{k},b_{h}^{k}) - [\mathbb{P}_{h}\underline{V}_{k,h+1}](s_{h}^{k},a_{h}^{k},b_{h}^{k})\notag\\
&\qquad- [\overline{V}_{k,h+1}(s_{h+1}^{k}) - \underline{V}_{k,h+1}(s_{h+1}^{k})]\Big]\notag\\
&\leq 2\beta^{(0)}_{k}\sum_{k=1}^{K}\sum_{h=1}^{H}\overline{\sigma}_{k,h}\min\Big\{1, \Big\|\big[\overline{\bSigma}_{k,h}^{(0)}\big]^{-1/2}\bphi_{\overline{V}_{k,h+1}}/\overline{\sigma}_{k,h}\Big\|_{2}\Big\}\notag\\
&\qquad+ 2\beta^{(0)}_{k}\sum_{k=1}^{K}\sum_{h=1}^{H}\underline{\sigma}_{k,h}\min\Big\{1,\Big\|\big[\underline{\bSigma}_{k,h}^{(0)}\big]^{-1/2}\bphi_{\underline{V}_{k,h+1}}/\underline{\sigma}_{k,h}\Big\|_{2}\Big\} + 8H\sqrt{2T\log(H/\delta)}\notag\\
&\leq 2\beta^{(0)}_{k}\sqrt{\sum_{k=1}^{K}\sum_{h=1}^{H}\overline{\sigma}_{k,h}^2}\sqrt{\sum_{k=1}^{K}\sum_{h=1}^{H}\min\Big\{1,\Big\|\big[\overline{\bSigma}_{k,h}^{(0)}\big]^{-1/2}\bphi_{\overline{V}_{k,h+1}}/\overline{\sigma}_{k,h}\Big\|_{2}^{2}\Big\}}\notag\\
&\qquad+ 2\beta^{(0)}_{k}\sqrt{\sum_{k=1}^{K}\sum_{h=1}^{H}\underline{\sigma}_{k,h}^2}\sqrt{\sum_{k=1}^{K}\sum_{h=1}^{H}\min\Big\{1,\Big\|\big[\underline{\bSigma}_{k,h}^{(0)}\big]^{-1/2}\bphi_{\underline{V}_{k,h+1}}/\underline{\sigma}_{k,h}\Big\|_{2}^{2}\Big\}}\notag\\
&\qquad + 8H\sqrt{2T\log(H/\delta)}\notag\\
&\leq 2\beta^{(0)}_{K}\sqrt{\sum_{k=1}^{K}\sum_{h=1}^{H}\overline{\sigma}_{k,h}^2}\sqrt{2Hd\log(1+K/\lambda)}\notag\\
&\qquad + 2\beta^{(0)}_{K}\sqrt{\sum_{k=1}^{K}\sum_{h=1}^{H}\underline{\sigma}_{k,h}^2}\sqrt{2Hd\log(1+K/\lambda)} + 8H\sqrt{2T\log(H/\delta)}\notag \\
&\leq  4\beta^{(0)}_{K}\sqrt{\sum_{k=1}^{K}\sum_{h=1}^{H}\overline{\sigma}_{k,h}^2 + \underline{\sigma}_{k,h}^2}\sqrt{2Hd\log(1+K/\lambda)}+ 8H\sqrt{2T\log(H/\delta)},
\end{align}
where the first inequality holds since $\overline{V}_{k, H+1} = \underline{V}_{k, H+1} = 0$, the second inequality holds on event $\mathcal{E}_{1}$, the third inequality holds due to Cauchy-Schwarz inequality, the fourth inequality holds due to Azuma Hoeffding inequality with the fact that $\Big\|\bphi_{\overline{V}_{k,h+1}}(s_{h}^{k},a_{h}^{k},b_{h}^{k})/\overline{\sigma}_{k,h}\Big\|_{2}\leq \Big\|\bphi_{\overline{V}_{k,h+1}}(s_{h}^{k},a_{h}^{k},b_{h}^{k})\Big\|_{2}\cdot \sqrt{d}/H \leq \sqrt{d}$, $\Big\|\bphi_{\underline{V}_{k,h+1}}(s_{h}^{k},a_{h}^{k},b_{h}^{k})/\underline{\sigma}_{k,h}\Big\|_{2}\leq \Big\|\bphi_{\underline{V}_{k,h+1}}(s_{h}^{k},a_{h}^{k},b_{h}^{k})\Big\|_{2}\cdot \sqrt{d}/H \leq \sqrt{d}$, the last inequality is by the fact that $\sqrt{a} + \sqrt{b} \leq 2\sqrt{a+b}$. \eqref{summation1} holds for any $h'$, then we have following inequality holds
\begin{align*}
&\sum_{k=1}^{K}\sum_{h=1}^{H}\mathbb{P}_{h}[\overline{V}_{k,h+1} - \underline{V}_{k,h+1}](s_{h}^{k},a_{h}^{k},b_{h}^{k})\\
&= \sum_{k=1}^{K}\sum_{h=1}^{H}[\overline{V}_{k,h} - \underline{V}_{k,h}](s_{h}^{k}) + \sum_{k=1}^{K}\sum_{h=1}^{H}\Big[[\mathbb{P}_{h}\overline{V}_{k,h+1}](s_{h}^{k},a_{h}^{k},b_{h}^{k}) - [\mathbb{P}_{h}\underline{V}_{k,h+1}](s_{h}^{k},a_{h}^{k},b_{h}^{k})\notag\\
&\qquad- [\overline{V}_{k,h+1}(s_{h+1}^{k}) - \underline{V}_{k,h+1}(s_{h+1}^{k})]\Big]\\
&\leq 4\beta^{(0)}_{K}\sqrt{\sum_{k=1}^{K}\sum_{h=1}^{H}\overline{\sigma}_{k,h}^2+\underline{\sigma}_{k,h}^2}\sqrt{2H^{3}d\log(1+K/\lambda)} + 8H^{2}\sqrt{2T\log(H/\delta)},
\end{align*}
where the inequality holds due to \eqref{summation1} and on event $\event_{1}$.
\end{proof}

\subsection{Proof of Lemma \ref{Lemma:VARBOUND}}
To estimate the variance in weighted ridge regression we need the following lemma, which is similar to the Lemma  \ref{Lemma:ULCB} but without tolerant error $\epsilon$.
\begin{lemma}\label{Lemma:ULCB2}
Suppose the event $\mathcal{E}$ hold. Then we have for any $s,a,b,k,h$, the following inequalities hold,
\begin{align*}
\underline{Q}_{k,h}(s,a,b) \leq Q_{h}^{\mu^{k}}(s,a,b) \leq \overline{Q}_{k,h}(s,a,b),
\end{align*}
and 
\begin{align*}
\underline{V}_{k,h}(s) \leq V_{h}^{\mu^{k}}(s)  \leq \overline{V}_{k,h}(s).  
\end{align*}
\end{lemma}
\begin{proof}
For simplicity, we only prove the following UCB by induction,
\begin{align*}
Q_{h}^{\mu^{k}}(s,a,b) \leq \overline{Q}_{k,h}(s,a,b), V_{h}^{\mu^{k}}(s) \leq \overline{V}_{k,h}(s).
\end{align*}
The base case $h = H + 1$ holds trivially since the terminal cost is zero. Now we assume that the bounds \eqref{eq: UCB} holds for step $h+1$. That is, 
\begin{align*}
Q_{h+1}^{\mu^{k}}(s,a,b) \leq \overline{Q}_{k,h+1}(s,a,b), V_{h+1}^{\mu^{k}}(s) \leq \overline{V}_{k,h+1}(s). 
\end{align*}
If $\overline{Q}_{k,h}(s,a,b) \geq H$, then it is obvious to have $Q_{h}^{\mu^{k}}(s,a,b) \leq \overline{Q}_{k,h}(s,a,b)$, otherwise we have that
\begin{align}
&\overline{Q}_{k,h}(s,a,b) -  Q_{h}^{\mu^{k}}(s,a,b)\notag\\
&=  \langle \overline{\btheta}^{(0)}_{k,h}, \bphi_{\overline{V}_{k,h+1}} \rangle + \beta^{(0)}_{k}\Big\|\big[\overline{\bSigma}_{k,h}^{(0)}\big]^{-1/2}\bphi_{\overline{V}_{k,h+1}}\Big\|_{2} - \langle \btheta_{h}^{*}, \bphi_{\overline{V}_{k,h+1}} \rangle \notag\\
&\qquad+ \mathbb{P}_{h}\overline{V}_{k,h+1}(s,a,b) - \mathbb{P}_{h}V_{h+1}^{*,\nu^{k}}(s)\notag\\
&\geq \beta_{k}^{(0)}\Big\|\big[\overline{\bSigma}_{k,h}^{(0)}\big]^{-1/2}\bphi_{\overline{V}_{k,h+1}}\Big\|_{2} - \Big\|\big[\overline{\bSigma}_{k,h}^{(0)}\big]^{1/2}(\overline{\btheta}^{(0)}_{k,h} - \btheta_{h}^{*})\Big\|_{2}\Big\|\big[\overline{\bSigma}_{k,h}^{(0)}\big]^{-1/2}\bphi_{\overline{V}_{k,h+1}}\Big\|_{2}\notag\\
&\qquad+ \mathbb{P}_{h}\overline{V}_{k,h+1}(s) - \mathbb{P}_{h}V_{h+1}^{\mu^{k}}(s)\notag\\
&\geq \mathbb{P}_{h}\overline{V}_{k,h+1}(s) - \mathbb{P}_{h}V_{h+1}^{\mu^{k}}(s)\notag\\
&\geq 0\label{eq:QQ},
\end{align}
where the first inequality holds due to Cauchy-Schwarz inequality, the second inequality holds since
the assumption that $\theta^{*}_{h} \in \overline{\cC}^{(0)}_{k,h}$ in event $\mathcal{E}$, the third inequality holds by the induction assumption.

Then we have that 
\begin{align*}
\overline{V}_{k,h}(s)&= \EE_{(a,b)\sim \mu_{h}^{k}(\cdot,\cdot|s)}\big[\overline{Q}_{k,h}(s,a,b)\big]\\
&\geq \EE_{(a,b)\sim \mu_{h}^{k}(\cdot,\cdot|s)}\big[Q_{h}^{\mu^{k}}(s,a',b)\big]\\
&= V_{h}^{\mu^{k}}(s),
\end{align*}
where the inequality is by \eqref{eq:QQ}, the last inequality is due to the Bellman equation. Therefore,
our proof is completed.
\end{proof}
\begin{lemma}\label{Lemma:assist2}(Lemma 11, \citealt{abbasi2011improved}). For any $\{\xb_{t}\}_{t=1}^{T} \subset \mathbb{R}^d$ satisfying that $\|\xb_{t}\|_{2}\leq L$, let $\Ab_{0} = \lambda \Ib$ and $\Ab_{t} = \Ab_{0} + \sum_{i=1}^{t}\xb_{i}\xb_{i}^{\top}$, then we have 
\begin{align*}
\sum_{t=1}^{T}\min \{1, \|\xb_{t}\|^{2}_{\Ab_{t-1}^{-1}}\} \leq 2d\log\frac{d\lambda + TL^{2}}{d\lambda}.
\end{align*}
\end{lemma}
\begin{proof}[Proof of Lemma \ref{Lemma:VARBOUND}]
Suppose the event in Lemma 9.1 holds, we have the following results:
\begin{align}
\sum_{k=1}^{K}\sum_{h=1}^{H}\overline{\sigma}_{k,h}^2 &= \sum_{k=1}^{K}\sum_{h=1}^{H} \big[H^{2}/d + \mathbb{V}^{\text{est}}_{k,h}\overline{V}_{k,h+1}(s_{h}^{k}, a_{h}^{k}, b_{h}^{k}) + \overline{E}_{k,h} \big]\notag\\
&= H^{2}T/d + \underbrace{\sum_{k=1}^{K}\sum_{h=1}^{H}\big[\mathbb{V}_{h}\overline{V}_{k,h+1}(s_{h}^{k}, a_{h}^{k},b_{h}^{k}) - \mathbb{V}_{h}V_{h+1}^{\mu^{k}}(s_{h}^{k}, a_{h}^{k},b_{h}^{k})\big]}_{I_{1}}\notag\\
&\qquad+ \underbrace{2\sum_{k=1}^{K}\sum_{h=1}^{H}\overline{E}_{k,h}}_{I_{2}}+\underbrace{\sum_{k=1}^{K}\sum_{h=1}^{H}\mathbb{V}_{h}V_{h+1}^{\mu^{k}}(s_{h}^{k}, a_{h}^{k},b_{h}^{k})}_{I_{3}}\notag\\
&\qquad+  \underbrace{\sum_{k=1}^{K}\sum_{h=1}^{H}\big[ \mathbb{V}_{h}^{\text{est}}\overline{V}_{k,h+1}(s_{h}^{k}, a_{h}^{k},b_{h}^{k})- \mathbb{V}_{h}\overline{V}_{k,h+1}(s_{h}^{k}, a_{h}^{k},b_{h}^{k}) - \overline{E}_{k,h}\big]}_{I_{4}}\label{eq:VARBOUND1},
\end{align}
where the first equation is by the definition of $\overline{\sigma}_{k,h}$. To bound $I_{1}$, we have 
\begin{align*}
I_{1} &= \sum_{k=1}^{K}\sum_{h=1}^{H}\big[\mathbb{P}_{h}\overline{V}_{k,h+1}^{2}(s_{h}^{k}, a_{h}^{k},b_{h}^{k}) - \mathbb{P}_{h}[V_{h+1}^{\mu^{k}}]^2(s_{h}^{k}, a_{h}^{k},b_{h}^{k})\big]\\
&\qquad- \sum_{k=1}^{K}\sum_{h=1}^{H}\big[[\mathbb{P}_{h}\overline{V}_{k,h+1}]^{2}(s_{h}^{k}, a_{h}^{k},b_{h}^{k}) - [\mathbb{P}_{h}V_{h+1}^{\mu^{k}}]^2(s_{h}^{k}, a_{h}^{k},b_{h}^{k})\big]\\
&\leq \sum_{k=1}^{K}\sum_{h=1}^{H}\mathbb{P}_{h}[(\overline{V}_{k,h+1} - V_{h+1}^{\mu^{k}})(\overline{V}_{k,h+1} + V_{h+1}^{\mu^{k}})](s_{h}^{k}, a_{h}^{k},b_{h}^{k})\\
&\qquad- \sum_{k=1}^{K}\sum_{h=1}^{H}[(\mathbb{P}_{h}\overline{V}_{k,h+1} - \mathbb{P}_{h} V_{h+1}^{\mu^{k}} )(\mathbb{P}_{h}\overline{V}_{k,h+1}+ \mathbb{P}_{h}V_{h+1}^{\mu^{k}})](s_{h}^{k}, a_{h}^{k},b_{h}^{k})\\
&\leq 4H\sum_{k=1}^{K}\sum_{h=1}^{H}\mathbb{P}_{h}[\overline{V}_{k,h+1} - V_{h+1}^{\mu^{k}}](s_{h}^{k}, a_{h}^{k},b_{h}^{k}) \\
&=  4H\sum_{k=1}^{K}\sum_{h=1}^{H}\mathbb{P}_{h}[\overline{V}_{k,h+1} - V_{h+1}^{\mu^{k}}](s_{h}^{k}, a_{h}^{k},b_{h}^{k}),
\end{align*}
where the first inequality is by $|\overline{V}_{k,h+1}|, |V_{h+1}^{\mu^{k}}|\leq H$, and the second inequality is by $\overline{V}_{k,h+1} - V_{h+1}^{\mu^{k}} \geq 0$ due to Lemma \ref{Lemma:ULCB2}.
To bound $I_{2}$, we have 
\begin{align*}
I_{2} &\leq 2\sum_{k=1}^{K}\sum_{h=1}^{H}\beta_{k}^{(1)}\min\big\{1, \Big\|\big[\overline{\bSigma}_{k, h}^{(1)}\big]^{-1/2}\bphi_{\overline{V}_{k, h+1}^2}(s_{h}^{k}, a_{h}^{k},b_{h}^{k})\Big\|_{2}\big\}\\
&\qquad + 4H\sum_{k=1}^{K}\sum_{h=1}^{H}\beta_{k}^{(2)}\bar{\sigma}_{k,h}\min\big\{1,\Big\|\big[\overline{\bSigma}_{k,h}^{(0)}\big]^{-1/2}\bphi_{\overline{V}_{k, h+1}}(s_{h}^{k}, a_{h}^{k},b_{h}^{k})/\bar{\sigma}_{k,h}\Big\|_{2}\big\}\\
&\leq 2\beta_{K}^{(1)}\sqrt{T}\sqrt{\sum_{k=1}^{K}\sum_{h=1}^{H}\min\big\{1, \Big\|\big[\overline{\bSigma}_{k, h}^{(1)}\big]^{-1/2}\bphi_{\overline{V}_{k, h+1}^2}(s_{h}^{k}, a_{h}^{k},b_{h}^{k})\Big\|_{2}^2\big\}}\\
&\qquad + 7\beta_{K}^{(1)}H^{2}\sqrt{T}\sqrt{\sum_{k=1}^{K}\sum_{h=1}^{H}\min\big\{1,\Big\|\big[\overline{\bSigma}_{k, h}^{(1)}\big]^{-1/2}\bphi_{\overline{V}_{k, h+1}}(s_{h}^{k}, a_{h}^{k},b_{h}^{k})\Big\|_{2}^2/\bar{\sigma}\big\}}\\
&\leq 2\beta_{K}^{(2)}\sqrt{T}\sqrt{2dH\log(1+KH^{4}/(d\lambda))} + 7\beta_{K}^{(1)}H^{2}\sqrt{T}\sqrt{2dH\log(1+K/\lambda)},
\end{align*}
where the first inequality holds due to $\beta_{k}^{(1)}\geq H^{2}$ and $\beta_{k}^{(2)}\bar{\sigma}_{k,h}\geq \sqrt{d}\cdot H/\sqrt{d} = H$, the second inequality holds due to Cauchy-Schwartz inequality, $\beta_{k}^{(1)} \leq \beta_{K}^{(1)}$, $\beta_{k}^{(2)} \leq \beta_{K}^{(2)}$,
\begin{align*}
\overline{\sigma}_{k,h}^{2} = \max\{H^{2}/d, \mathbb{V}^{\text{est}}\overline{V}_{k,h+1}(s_{h}^{k},a_{h}^{k},b_{h}^{k}) + \overline{E}_{k,h}\} \leq \max\{H^{2}/d, H^{2}+2H^{2}\}=3H^{2},
\end{align*}
% $\bar{\sigma}_{k,h}^{2} \leq 3H^{2}$, 
the third inequality holds due to Lemma \ref{Lemma:assist2}. Next we bound $I_{3}$, since event $\mathcal{E}_{2}$ holds , we have
\begin{align*}
I_{3} \leq 3(HT + H^{3}\log(1/\delta)).
\end{align*}
Finally, due to on event $\event$, we have $I_{4}\leq 0$. We finish the proof by substituting $I_{1}, I_{2}, I_{3}, I_{4}$ into \eqref{eq:VARBOUND1}.
\end{proof}

\section{Full Version of Algorithm \ref{alg:Offline}}\label{sec:Full Algorithm}

In this section, we present the full version of Algorithm \ref{alg:Offline} in Algorithm \ref{alg:fullalg}.

\begin{algorithm}[ht!]
\caption{ $\algname$}\label{alg:fullalg}
\begin{algorithmic}[1]
\STATE \textbf{Input:} Regularization parameter $\lambda$, Number of episode $K$, number of horizon $H$.
% \STATE receives $s_{1}$.
\STATE For any $h$, $\overline{\bSigma}^{(i)}_{1,h}  \leftarrow \underline{\bSigma}^{(i)}_{1,h} \leftarrow \lambda \Ib$; $\overline{\bbb}^{(i)}_{1,h} \leftarrow \underline{\bbb}^{(i)}_{1,h} \leftarrow \zero$; $\overline{\btheta}^{(i)}_{1,h} \leftarrow \underline{\btheta}^{(i)}_{1,h} \leftarrow  \zero$, for $i \in \{0,1\}$.  
\FOR{$k = 1,\ldots, K$}
\STATE $\overline{V}_{k,H+1}(\cdot) \leftarrow 0$, $\underline{V}_{k,H+1}(\cdot) \leftarrow 0$
\FOR{$h = H, \ldots, 1$}
\STATE Set  $\overline{Q}_{k,h}(\cdot,\cdot,\cdot)$ and $\underline{Q}_{k,h}(\cdot,\cdot,\cdot)$ as in \eqref{eq:overunderq}.
\FOR {$s\in \cS$}
\STATE Let $\mu_{h}^{k}(\cdot,\cdot|s) = \epsilon\text{-CCE}(\overline{Q}_{k,h}(s,\cdot,\cdot),\underline{Q}_{k,h}(s,\cdot,\cdot))$. \label{line:1}
\STATE $\overline{V}_{k,h}(s) = \EE_{(a,b) \sim \mu_{h}^{k}(\cdot,\cdot|s)}\overline{Q}_{k,h}(s, a, b)$, $\underline{V}_{k,h}(s) = \EE_{(a,b) \sim \mu_{h}^{k}(\cdot,\cdot|s)}\underline{Q}_{k,h}(s, a, b)$
\STATE $\pi_{h}^{k}(\cdot|s) = \cP_{\max}\mu_{h}^{k}(\cdot,\cdot|s)$, $\nu_{h}^{k}(\cdot|s) = \cP_{\min}\mu_{h}^{k}(\cdot,\cdot|s)$
\ENDFOR
\ENDFOR
\STATE receives $s_{1}^{k}$
\FOR{$h = 1, \ldots, H$}
\STATE Take action $a_{h}^{k} \sim \pi_{h}^{k}(s_{h}^{k})$ and $b_{h}^{k} \sim \nu_{h}^{k}(s_{h}^{k})$ and receives $s_{h+1}^{k} \sim \mathbb{P}(\cdot | s_{h}^{k}, a_{h}^{k}, b_{h}^{k})$.
\STATE Set $\mathbb{V}^{\text{est}}\overline{V}_{k,h+1}(s_{h}^{k},a_{h}^{k},b_{h}^{k})$ and $\mathbb{V}^{\text{est}}\underline{V}_{k,h+1}(s_{h}^{k},a_{h}^{k},b_{h}^{k})$ as in \eqref{eq:overunderv}.
\STATE Set $\overline{E}_{k, h}, \underline{E}_{k, h}, \overline{\bsigma}_{k,h}, \underline{\bsigma}_{k,h}, \overline{\bSigma}_{k+1,h}^{(0)}, \underline{\bSigma}_{k+1,h}^{(0)}$, $\overline{\bbb}_{k+1,h}^{(0)}, \underline{\bbb}_{k+1,h}^{(0)}, \overline{\bSigma}_{k+1,h}^{(1)}, \underline{\bSigma}_{k+1,h}^{(1)}, \overline{\bbb}_{k+1,h}^{(1)}, \underline{\bbb}_{k+1,h}^{(1)}$ as defined in \eqref{summ2}.
\STATE Set $\overline{\btheta}_{k+1,h}^{(i)} \leftarrow \big[\overline{\bSigma}_{k+1,h}^{(i)}\big]^{-1}\overline{\bbb}_{k+1,h}^{(i)} $,  $\underline{\btheta}_{k+1,h}^{(i)} \leftarrow \big[\underline{\bSigma}_{k+1,h}^{(i)}\big]^{-1}\underline{\bbb}_{k+1,h}^{(i)} $, $i=0,1$
\ENDFOR
\ENDFOR
\end{algorithmic}
\end{algorithm}

%For Algorithm \ref{alg:fullalg}, several parameters are updated as follows:
\noindent\textbf{Update of optimistic action-value function:}
\begin{align}
 &\overline{Q}_{k,h}(\cdot,\cdot,\cdot) \leftarrow \Big[ r_{h}(\cdot,\cdot,\cdot) + \langle \overline{\btheta}^{(0)}_{k,h}, \bphi_{\overline{V}_{k,h+1}}(\cdot,\cdot,\cdot) \rangle + \beta^{(0)}_{k}\Big\|\big[\overline{\bSigma}_{k,h}^{(0)}\big]^{-1/2}\bphi_{\overline{V}_{k,h+1}}(\cdot,\cdot,\cdot)\Big\|_{2}\Big]_{[-H,H]}\notag\\
&\underline{Q}_{k,h}(\cdot,\cdot,\cdot) \leftarrow \Big[ r_{h}(\cdot, \cdot, \cdot) + \langle  \underline{\btheta}^{(0)}_{k,h}, \bphi_{\underline{V}_{k,h+1}}(\cdot,\cdot,\cdot) \rangle - \beta^{(0)}_{k}\Big\|\big[\underline{\bSigma}_{k,h}^{(0)}\big]^{-1/2}\bphi_{\underline{V}_{k,h+1}}(\cdot,\cdot,\cdot)\Big\|_{2}\Big]_{[-H,H]}. \label{eq:overunderq}   
\end{align}

\noindent\textbf{Update of variance estimation:}
\begin{align}
&\mathbb{V}^{\text{est}}\overline{V}_{k,h+1}(s_{h}^{k},a_{h}^{k},b_{h}^{k})\leftarrow \big[ \langle\bphi_{\overline{V}^{2}_{k,h+1}}(s_{h}^{k},a_{h}^{k},b_{h}^{k}), \overline{\btheta}^{(1)}_{k,h}\rangle \big]_{[0,H^{2}]} - \big[\langle\bphi_{\overline{V}_{k,h+1}}(s_{h}^{k},a_{h}^{k},b_{h}^{k}), \overline{\btheta}^{(0)}_{k,h}\rangle\big]^{2}_{[-H,H]},\notag \\
    & \mathbb{V}^{\text{est}}\underline{V}_{k,h+1}(s_{h}^{k},a_{h}^{k},b_{h}^{k})\leftarrow \big[ \langle\bphi_{\underline{V}^{2}_{k,h+1}}(s_{h}^{k},a_{h}^{k},b_{h}^{k}), \underline{\btheta}^{(1)}_{k,h}\rangle\big]_{[0,H^{2}]} - \big[ \langle\bphi_{\underline{V}_{k,h+1}}(s_{h}^{k},a_{h}^{k},b_{h}^{k}), \underline{\btheta}^{(0)}_{k,h}\rangle\big]^{2}_{[-H,H]}.\label{eq:overunderv}
\end{align}

\noindent\textbf{Update of other parameters:}
\begin{align}
 &\overline{E}_{k, h} = \min\big\{H^{2}, \beta^{(1)}_{k}\Big\|\big[\overline{\bSigma}_{k, h}^{(1)}\big]^{-1/2}\bphi_{\overline{V}_{k, h+1}^2}(s_{h}^{k}, a_{h}^{k},b_{h}^{k})\Big\|_{2}\big\}\notag\\
&\qquad + \min\big\{H^{2}, 2H\beta^{(2)}_{k}\Big\|\overline{\bSigma}_{k, h}^{(0)-1/2}\bphi_{\overline{V}_{k, h+1}}(s_{h}^{k}, a_{h}^{k},b_{h}^{k})\Big\|_{2}\big\}\notag \\
&\underline{E}_{k, h} = \min\big\{H^{2}, \beta^{(1)}_{k}\Big\|\big[\underline{\bSigma}_{k, h}^{(1)}\big]^{-1/2}\bphi_{\underline{V}_{k, h+1}^2}(s_{h}^{k}, a_{h}^{k},b_{h}^{k})\Big\|_{2}\big\}\notag\\
&\qquad + \min\big\{H^{2}, 2H\beta^{(2)}_{k}\Big\|\underline{\bSigma}_{k, h}^{(0)-1/2}\bphi_{\underline{V}_{k, h+1}}(s_{h}^{k}, a_{h}^{k},b_{h}^{k})\Big\|_{2}\big\},\notag \\
&\overline{\bsigma}_{k,h} = \sqrt{\max\{H^{2}/4d, \mathbb{V}^{\text{est}}\overline{V}_{k,h+1}(s_{h}^{k},a_{h}^{k},b_{h}^{k}) + \overline{E}_{k,h}\}},\notag \\
&\underline{\bsigma}_{k,h} = \sqrt{\max\{H^{2}/4d, \mathbb{V}^{\text{est}}\underline{V}_{k,h+1}(s_{h}^{k},a_{h}^{k},b_{h}^{k}) + \underline{E}_{k,h}\}},\notag \\
&\overline{\bSigma}_{k+1,h}^{(0)} \leftarrow \overline{\bSigma}_{k,h}^{(0)} + \overline{\bsigma}_{k,h}^{-2} \bphi_{\overline{V}_{k,h+1}}(s_{h}^{k},a_{h}^{k},b_{h}^{k})\bphi_{\overline{V}_{k,h+1}}(s_{h}^{k},a_{h}^{k},b_{h}^{k})^{\top},\notag \\
&\underline{\bSigma}_{k+1,h}^{(0)} \leftarrow \underline{\bSigma}_{k,h}^{(0)} + \underline{\bsigma}_{k,h}^{-2} \bphi_{\underline{V}_{k,h+1}}(s_{h}^{k},a_{h}^{k},b_{h}^{k})\bphi_{\underline{V}_{k,h+1}}(s_{h}^{k},a_{h}^{k},b_{h}^{k})^{\top},\notag \\
&\overline{\bbb}_{k+1,h}^{(0)} = \overline{\bbb}_{k,h}^{(0)} + \overline{\bsigma}_{k,h}^{-2} \bphi_{\overline{V}_{k,h+1}}(s_{h}^{k},a_{h}^{k},b_{h}^{k})\overline{V}_{k,h+1}(s_{h+1}^{k}),\notag \\
&\underline{\bbb}_{k+1,h}^{(0)} = \underline{\bbb}_{k,h}^{(0)} + \underline{\bsigma}_{k,h}^{-2} \bphi_{\underline{V}_{k,h+1}}(s_{h}^{k},a_{h}^{k},b_{h}^{k})\underline{V}_{k,h+1}(s_{h+1}^{k}),\notag \\
&\overline{\bSigma}_{k+1,h}^{(1)} \leftarrow \overline{\bSigma}_{k,h}^{(1)} +  \bphi_{\overline{V}_{k,h+1}^{2}}(s_{h}^{k},a_{h}^{k},b_{h}^{k})\bphi_{\overline{V}_{k,h+1}^{2}}(s_{h}^{k},a_{h}^{k},b_{h}^{k})^{\top},\notag \\
&\underline{\bSigma}_{k+1,h}^{(1)} \leftarrow \underline{\bSigma}_{k,h}^{(1)} +  \bphi_{\underline{V}_{k,h+1}^{2}}(s_{h}^{k},a_{h}^{k},b_{h}^{k})\bphi_{\underline{V}_{k,h+1}^{2}}(s_{h}^{k},a_{h}^{k},b_{h}^{k})^{\top},\notag \\
&\overline{\bbb}_{k+1,h}^{(1)} = \overline{\bbb}_{k,h}^{(1)} +  \bphi_{\overline{V}_{k,h+1}^{2}}(s_{h}^{k},a_{h}^{k},b_{h}^{k})\overline{V}_{k,h+1}^{2}(s_{h+1}^{k}),\notag \\
&\underline{\bbb}_{k+1,h}^{(1)} = \underline{\bbb}_{k,h}^{(1)} +  \bphi_{\underline{V}_{k,h+1}^{2}}(s_{h}^{k},a_{h}^{k},b_{h}^{k})\underline{V}_{k,h+1}^{2}(s_{h+1}^{k}).\label{summ2}
\end{align}
In line~\ref{line:1} of Algorithm~\ref{alg:fullalg}, we need to call the  $\epsilon\text{-CCE}$ subroutine. In detail, for any fixed state $s$, and two matrices $\overline{Q}_{k,h}(s,\cdot,\cdot)$ and $\underline{Q}_{k,h}(s,\cdot,\cdot) \in [0,1]^{|\cA_{\max}|\times |\cA_{\min}|}$, the subroutine $\epsilon\text{-CCE}(\cdot,\cdot)$ returns a distribution $\sigma \in \Delta_{|\cA_{\max}|\times |\cA_{\min}|}$ that satisfies 
\begin{align}
    \EE_{(a,b) \sim \sigma} \overline{Q}_{k,h}(s,a,b) \geq \max_{a' \in \cA_{\text{max}}}\EE_{(a,b) \sim \sigma }\overline{Q}_{k,h}(s,a',b)-\epsilon,\notag \\
    \EE_{(a,b) \sim \sigma} \underline{Q}_{k,h}(s,a,b) \leq \min_{b' \in \cA_{\text{min}}}\EE_{(a,b) \sim \sigma }\underline{Q}_{k,h}(s,a,b')+\epsilon.\label{eq:con}
\end{align}
$\eqref{eq:con}$ is a feasibility problem, where the constraints can be rewritten as $|\cA_{\max}| + |\cA_{\min}|$ linear
constraints on $\sigma \in \Delta_{|\cA_{\max}|\times |\cA_{\min}|}$. Thus it can be efficiently resolved by any linear programming algorithms. See also Appendix B in \citet{liu2020sharp}  and \citet{xie2020learning} for more detailed discussions.

\section{Extensions to Turn-based Games}\label{sec:turn}
In this section, we extend our algorithm and results to turn-based Markov games. 

\noindent\textbf{Turn-based MGs} A two-player zero-sum turn-based episodic MG is denoted by a tuple \\$M(\cS, \cA, H, \{r_h\}_{h=1}^H, \{\PP_h\}_{h=1}^H)$, where $\cS = \cS_{\text{max}} \cup \cS_{\text{min}}$, $\cS_{\text{max}}$ ($\cS_{\text{min}}$) are the states where the max (min)-player plays, $\cS_{\text{max}} \cap \cS_{\text{min}} = \emptyset$. Note that the partition of state space suggests that at each step, only one player can play. $\cA$ is the action space, $H$ is the length of game/episode, $r_h:\cS\times \cA \rightarrow [-1,1]$ is the reward function, $\PP_h(s'|s,a)$ denotes the transition probability for the max (min)-player ($s \in \cS_{\text{max}}$ or $\cS_{\text{max}}$) to take action $a$ and transit to next state $s'$. Similar to the linear mixture MGs, we can define linear mixture turn-based MGs as follows. %satisfy the following linear mixture turn-based MDP assumption:
\begin{definition}
$M(\cS, \cA, H, \{r_h\}_{h=1}^H, \{\PP_h\}_{h=1}^H)$ is called a time-inhomogeneous, episodic $B$-bounded linear mixture turn-based Markov game if there exist $\{\btheta_h\}_{h=1}^H \subset \RR^d$ and $\tilde\bphi(s'|s,a) \in \RR^d$ satisfying 
\begin{align}
    &\|\btheta_h\|_2 \leq B,\qquad \forall V:\cS \rightarrow [-1,1],\ \bigg\|\sum_{s' \in \cS}\tilde\bphi(s'|s,a)V(s')\bigg\|_2 \leq 1,\ \notag
\end{align}
such that $\mathbb{P}_{h}(s'|s,a) = \langle \bphi(s'|s,a), \btheta_{h}\rangle$ for any state-action-state triplet $(s, a, s')$ and any step $h$.
\end{definition}
Based on above definition, we show that any turn-based linear mixture MG can be regarded as a special case of linear mixture simultaneous-move MG. In fact, for any turn-based linear mixture MG with feature mapping $\tilde\bphi(\cdot|\cdot,\cdot)$ and reward $\tilde r_h(\cdot, \cdot)$, we can define the corresponding  linear mixture simultaneous-move MG with feature mapping $\bphi(\cdot|\cdot, \cdot, \cdot)$ and reward $r_h(\cdot, \cdot, \cdot)$ as follows: for each $s \in \cS_{\max}$, 
\begin{align*}
\bphi(s'|s,a,b) = \tilde{\bphi}_{h}(s'|s,a), \, r_{h}(s'|s,a,b) = \tilde{r}_{h}(s'|s,a),
\end{align*}
and for each $s \in  \cS_{\min}$, 
\begin{align*}
\bphi(s'|s,a,b) = \tilde{\bphi}_{h}(s'|s,b), \, r_{h}(s'|s,a,b) = \tilde{r}_{h}(s'|s,b).
\end{align*}
Therefore, we can still use Algorithm \ref{alg:Offline} to find the Nash equilibrium. Notice that for the turn-based game, at each step only one player can take action. Thus, the $\epsilon$-CCE routine in Line \ref{alg:CCE} of Algorithm \ref{alg:Offline} needs be replaced by two separate subroutines: taking $\pi_h^k$ and $\nu_h^k$ as greedy policies w.r.t. $\overline{Q}_{k,h}$ and  $\underline{Q}_{k,h}$. For completeness, we present the turn-based version of Algorithm \ref{alg:Offline} as Algorithm \ref{alg:turnbased}.

\begin{algorithm}[ht!]
\caption{Turn-based  \algname}\label{alg:turnbased}
\begin{algorithmic}[1]
\STATE For any $h$, $\overline{\bSigma}^{(i)}_{1,h}  \leftarrow \underline{\bSigma}^{(i)}_{1,h} \leftarrow \lambda \Ib$; $\overline{\bbb}^{(i)}_{1,h} \leftarrow \underline{\bbb}^{(i)}_{1,h} \leftarrow \zero$; $\overline{\btheta}^{(i)}_{1,h} \leftarrow \underline{\btheta}^{(i)}_{1,h} \leftarrow  \zero$, for $i \in \{0,1\}$.  
\FOR{$k = 1,\ldots, K$}
\STATE $\overline{V}_{k,H+1}(\cdot) \leftarrow 0$, $\underline{V}_{k,H+1}(\cdot) \leftarrow 0$
\FOR{$h = H, \ldots, 1$}
\STATE Set $\overline{Q}_{k,h}(\cdot,\cdot)$ and $\underline{Q}_{k,h}(\cdot,\cdot)$ as in \eqref{eq:turnoverunderq}.
\FOR {$s\in \cS_{\max}$}
\STATE $\pi_{h}^{k}(\cdot|s) = \max_{a\in \cA}\overline{Q}_{k,h}(s,a)$, $\overline{V}_{k,h}(s) = \EE_{a \sim \pi_{h}^{k}(\cdot|s)}\overline{Q}_{k,h}(s, a)$.
\ENDFOR
\FOR {$s\in \cS_{\min}$}
\STATE $\nu_{h}^{k}(\cdot|s) = \min_{b\in \cA}\underline{Q}_{k,h}(s,b)$, $\underline{V}_{k,h}(s) = \EE_{b \sim \nu_{h}^{k}(\cdot|s)}\underline{Q}_{k,h}(s, b)$.
\ENDFOR
\ENDFOR
\STATE receives $s_{1}^k$
\FOR{$h = 1, \ldots, H$}
\IF{$s_h^k \in \cS_{\max}$}
\STATE Take action $a_{h}^{k} \sim \pi_{h}^{k}(\cdot|s_{h}^{k})$ and 
receives $s_{h+1}^{k} \sim \mathbb{P}(\cdot | s_{h}^{k}, a_{h}^{k})$.
\ELSE
% \ELSIF{$s_h^k \in \cS_{\min}$}
\STATE Take action $a_{h}^{k} \sim \nu_{h}^{k}(\cdot|s_{h}^{k})$ and 
receives $s_{h+1}^{k} \sim \mathbb{P}(\cdot | s_{h}^{k}, a_{h}^{k})$.
\ENDIF
\STATE Set $\mathbb{V}^{\text{est}}\overline{V}_{k,h+1}(s_{h}^{k},a_{h}^{k})$ and $ \mathbb{V}^{\text{est}}\underline{V}_{k,h+1}(s_{h}^{k},a_{h}^{k})$ as in \eqref{eq:turnoverunderv}.

\STATE Set $\overline{E}_{k, h}, \underline{E}_{k, h},\overline{\bsigma}_{k,h}, \underline{\bsigma}_{k,h}$, $\overline{\bSigma}_{k+1,h}^{(0)}$, $\underline{\bSigma}_{k+1,h}^{(0)}$, $\overline{\bbb}_{k+1,h}^{(0)}$, $\underline{\bbb}_{k+1,h}^{(0)}$,  $\overline{\bSigma}_{k+1,h}^{(1)}$, $\underline{\bSigma}_{k+1,h}^{(1)}$, $\overline{\bbb}_{k+1,h}^{(1)}$, $\underline{\bbb}_{k+1,h}^{(1)}$ as defined in \eqref{summ1}.
\STATE Set $\overline{\btheta}_{k+1,h}^{(i)} \leftarrow \big[\overline{\bSigma}_{k+1,h}^{(i)}\big]^{-1}\overline{\bbb}_{k+1,h}^{(i)} $,  $\underline{\btheta}_{k+1,h}^{(i)} \leftarrow \big[\underline{\bSigma}_{k+1,h}^{(i)}\big]^{-1}\underline{\bbb}_{k+1,h}^{(i)} $, $i=0,1$
\ENDFOR
\ENDFOR
\end{algorithmic}
\end{algorithm}

%\newpage
\noindent\textbf{Update of optimistic action-value function:}
\begin{align}
&\overline{Q}_{k,h}(\cdot,\cdot) \leftarrow \min\{H, \tilde{r}_{h}(\cdot,\cdot) + \langle \overline{\btheta}^{(0)}_{k,h}, \tilde{\bphi}_{\overline{V}_{k,h+1}}(\cdot,\cdot) \rangle + \beta^{(0)}_{k}\Big\|\big[\overline{\bSigma}_{k,h}^{(0)}\big]^{-1/2}\tilde{\bphi}_{\overline{V}_{k,h+1}}(\cdot,\cdot)\Big\|_{2}\}\notag\\ &\underline{Q}_{k,h}(\cdot,\cdot) \leftarrow \max\{-H, \tilde{r}_{h}(\cdot,\cdot) + \langle \underline{\btheta}^{(0)}_{k,h}, \tilde{\bphi}_{\underline{V}_{k,h+1}}(\cdot,\cdot) \rangle - \beta^{(0)}_{k}\Big\|\big[\underline{\bSigma}_{k,h}^{(0)}\big]^{-1/2}\tilde{\bphi}_{\underline{V}_{k,h+1}}(\cdot,\cdot)\Big\|_{2}\} \label{eq:turnoverunderq}.    
\end{align}
\noindent\textbf{Update of variance estimation:}
\begin{align}
 &\mathbb{V}^{\text{est}}\overline{V}_{k,h+1}(s_{h}^{k},a_{h}^{k})\leftarrow \big[ \langle\tilde{\bphi}_{\overline{V}^{2}_{k,h+1}}(s_{h}^{k},a_{h}^{k}), \overline{\btheta}^{(1)}_{k,h}\rangle \big]_{[0,H^{2}]} - \big[\langle\tilde{\bphi}_{\overline{V}_{k,h+1}}(s_{h}^{k},a_{h}^{k}), \overline{\btheta}^{(0)}_{k,h}\rangle\big]^{2}_{[-H,H]},\notag \\
    & \mathbb{V}^{\text{est}}\underline{V}_{k,h+1}(s_{h}^{k},a_{h}^{k})\leftarrow \big[ \langle\tilde{\bphi}_{\underline{V}^{2}_{k,h+1}}(s_{h}^{k},a_{h}^{k}), \underline{\btheta}^{(1)}_{k,h}\rangle\big]_{[0,H^{2}]} - \big[ \langle\tilde{\bphi}_{\underline{V}_{k,h+1}}(s_{h}^{k},a_{h}^{k}), \underline{\btheta}^{(0)}_{k,h}\rangle\big]^{2}_{[-H,H]}.\label{eq:turnoverunderv}
\end{align}
\noindent\textbf{Update of other parameters:}
\begin{align}
& \overline{E}_{k, h} = \min\big\{H^{2}, \beta^{(1)}_{k}\Big\|\big[\overline{\bSigma}_{k, h}^{(1)}\big]^{-1/2}\tilde{\bphi}_{\overline{V}_{k, h+1}^2}(s_{h}^{k}, a_{h}^{k})\Big\|_{2}\big\}\notag\\
&\qquad + \min\big\{H^{2}, 2H\beta^{(2)}_{k}\Big\|\overline{\bSigma}_{k, h}^{(0)-1/2}\tilde{\bphi}_{\overline{V}_{k, h+1}}(s_{h}^{k}, a_{h}^{k})\Big\|_{2}\big\},\notag \\
&\underline{E}_{k, h} = \min\big\{H^{2}, \beta^{(1)}_{k}\Big\|\big[\underline{\bSigma}_{k, h}^{(1)}\big]^{-1/2}\tilde{\bphi}_{\underline{V}_{k, h+1}^2}(s_{h}^{k}, a_{h}^{k})\Big\|_{2}\big\}\notag\\
&\qquad + \min\big\{H^{2}, 2H\beta^{(2)}_{k}\Big\|\underline{\bSigma}_{k, h}^{(0)-1/2}\tilde{\bphi}_{\underline{V}_{k, h+1}}(s_{h}^{k}, a_{h}^{k})\Big\|_{2}\big\},\notag\\
&\overline{\bsigma}_{k,h} = \sqrt{\max\{H^{2}/d, \mathbb{V}^{\text{est}}\overline{V}_{k,h+1}(s_{h}^{k},a_{h}^{k}) + \overline{E}_{k,h}\}},\notag \\
& \underline{\bsigma}_{k,h} = \sqrt{\max\{H^{2}/d, \mathbb{V}^{\text{est}}\underline{V}_{k,h+1}(s_{h}^{k},a_h^k) + \underline{E}_{k,h}\}},\notag \\
    &\overline{\bSigma}_{k+1,h}^{(0)} \leftarrow \overline{\bSigma}_{k,h}^{(0)} + \overline{\bsigma}_{k,h}^{-2} \tilde{\bphi}_{\overline{V}_{k,h+1}}(s_{h}^{k},a_{h}^{k})\tilde{\bphi}_{\overline{V}_{k,h+1}}(s_{h}^{k},a_{h}^{k})^{\top}\notag \\
    &\underline{\bSigma}_{k+1,h}^{(0)} \leftarrow \underline{\bSigma}_{k,h}^{(0)} + \underline{\bsigma}_{k,h}^{-2} \tilde{\bphi}_{\underline{V}_{k,h+1}}(s_{h}^{k},b_{h}^{k})\tilde{\bphi}_{\underline{V}_{k,h+1}}(s_{h}^{k},a_{h}^{k})^{\top}\notag \\
    &\overline{\bbb}_{k+1,h}^{(0)} = \overline{\bbb}_{k,h}^{(0)} + \overline{\bsigma}_{k,h}^{-2} \tilde{\bphi}_{\overline{V}_{k,h+1}}(s_{h}^{k},a_{h}^{k})\overline{V}_{k,h+1}(s_{k,h+1})\notag \\
    & \underline{\bbb}_{k+1,h}^{(0)} = \underline{\bbb}_{k,h}^{(0)} + \underline{\bsigma}_{k,h}^{-2} \tilde{\bphi}_{\underline{V}_{k,h+1}}(s_{h}^{k},a_{h}^{k})\underline{V}_{k,h+1}(s_{k,h+1})\notag \\
    & \overline{\bSigma}_{k+1,h}^{(1)} \leftarrow \overline{\bSigma}_{k,h}^{(1)} +  \tilde{\bphi}_{\overline{V}_{k,h+1}^{2}}(s_{h}^{k},a_{h}^{k})\tilde{\bphi}_{\overline{V}_{k,h+1}^{2}}(s_{h}^{k},a_{h}^{k})^{\top}\notag \\
    & \underline{\bSigma}_{k+1,h}^{(1)} \leftarrow \underline{\bSigma}_{k,h}^{(1)} +  \tilde{\bphi}_{\underline{V}_{k,h+1}^{2}}(s_{h}^{k},a_h^k)\tilde{\bphi}_{\underline{V}_{k,h+1}^{2}}(s_{h}^{k},a_h^k)^{\top}\notag \\
    & \overline{\bbb}_{k+1,h}^{(1)} = \overline{\bbb}_{k,h}^{(1)} +  \tilde{\bphi}_{\overline{V}_{k,h+1}^{2}}(s_{h}^{k},a_{h}^{k})\overline{V}_{k,h+1}^{2}(s_{h+1}^{k}),\notag \\
    & \underline{\bbb}_{k+1,h}^{(1)} = \underline{\bbb}_{k,h}^{(1)} +  \tilde{\bphi}_{\underline{V}_{k,h+1}^{2}}(s_{h}^{k},a_h^k)\underline{V}_{k,h+1}^{2}(s_{h+1}^{k})
    .\label{summ1}
\end{align}
By Theorem \ref{mainthm}, we immediately have that the regret of our turn-based algorithm in Algorithm~\ref{alg:turnbased} is also bounded by
\begin{align*}
\tilde{\cO}(\sqrt{d^{2}H^{2}+dH^{3}}\sqrt{T}+d^{2}H^{3}+d^{3}H^{2}),
\end{align*}
where $T=KH$. Similarly, we can show that if $d\geq H$ and $T \geq d^{4}H^{2}$, our turn-based algorithm is nearly minimax optimal.

%Instead of dividing by state \citep{cui2020minimax, xie2020learning}, \citet{bai2020provable} consider dividing the steps $[H]$ into $\cH_{\max}$ and $\cH_{\min}$. This setting looks different at first glance. However, it can actually be written as the turn-based MGs in this section by augmenting the state $s$ by a latent variable $h$, i.e., $\tilde{s} = (s, h)$, to indicate the max-player's step and the min-player's step. 

\end{document}